\documentclass{article}

\PassOptionsToPackage{numbers, compress}{natbib}

\usepackage[final,main]{neurips_2026}

\usepackage[utf8]{inputenc} 
\usepackage[T1]{fontenc}    

\usepackage{amsfonts}       
\usepackage{amsmath,amssymb,amsthm,mathtools}

\usepackage{url}            
\usepackage{nicefrac}       
\usepackage{microtype}      

\usepackage{booktabs}       
\usepackage{multirow}
\usepackage{siunitx}
\usepackage{algorithm}
\usepackage{algorithmic}
\usepackage{enumitem}

\usepackage[table]{xcolor}
\usepackage{graphicx}
\usepackage{subcaption}
\usepackage{wrapfig}
\usepackage{tikz}
\usepackage{pgfplots}
\usepackage{pgfplotstable}

\usepackage{tabularx}
\usepackage{array}
\usepackage{listings}
\usepackage[most]{tcolorbox}

\usepackage{adjustbox}
\usepackage{caption}

\usetikzlibrary{arrows.meta,positioning,calc,decorations.pathreplacing}
\tcbuselibrary{skins,breakable}
\renewcommand{\paragraph}[1]{\vspace{.2em}\noindent\textbf{#1}}

\usepackage[pagebackref,breaklinks,colorlinks,allcolors=nipsred]{hyperref}

\pgfplotsset{compat=1.18}
\usetikzlibrary{patterns}

\definecolor{speedupcolor}{RGB}{236, 112, 99}
\definecolor{nfecolor}{RGB}{46, 134, 193}
\definecolor{nipsred}{rgb}{0.65,0.17,0.17}

\newtheorem{theorem}{Theorem}
\newtheorem{corollary}[theorem]{Corollary}
\newtheorem{lemma}[theorem]{Lemma}
\newtheorem{proposition}[theorem]{Proposition}
\newtheorem{remark}[theorem]{Remark}

\newcommand{\mask}{\texttt{[M]}}
\newcommand{\TV}{\mathrm{TV}}
\newcommand{\KL}{\mathrm{KL}}
\DeclareMathOperator*{\argmax}{arg\,max}

\newcommand{\seqx}{\boldsymbol{x}} 
\newcommand{\E}{\mathbb{E}}

\newcommand{\speedup}[1]{\textcolor{speedupcolor}{\tiny #1$\times$}}
\newcommand{\nfered}[1]{\textcolor{nfecolor}{\tiny $\downarrow$#1\%}}

\title{Fast-dLLM++: Fr\'{e}chet Profile Decoding for \\ Faster Diffusion LLM Inference}

\author{%
  Siva Rajesh Kasa$^{1*}$ \quad Yasong Dai$^{2*}$ \quad Sumit Negi$^{1}$ \quad Hongdong Li$^{1,2}$ \\[0.4em]
  $^{1}$Amazon Inc. \qquad $^{2}$School of Computing, Australian National University \\[0.2em]
  \texttt{kasa@u.nus.edu} \qquad \texttt{yasong.dai@anu.edu.au}
}

\begin{document}

\maketitle
\renewcommand{\thefootnote}{\fnsymbol{footnote}}\footnotetext[1]{Equal contribution.}\renewcommand{\thefootnote}{\arabic{footnote}}

\begin{abstract}
Diffusion large language models promise parallel token generation, yet inference remains bottlenecked by deciding which masked tokens can be safely committed together. Fast-dLLM addressed this with KV caching and confidence-guided parallel decoding, but its decoding theory uses a homogeneous high-confidence assumption that effectively reduces each candidate set to its weakest selected token. We argue that this leaves speed on the table because real decoding steps exhibit heterogeneous confidence profiles.
We propose \textbf{Fast-dLLM++}, a training-free extension that introduces \textbf{Fr\'{e}chet profile decoding}: selecting parallel commit sets from the full sorted confidence profile rather than a single worst-case confidence. The resulting rule is a heterogeneous-confidence generalization of Fast-dLLM's factor selector and it recovers the previous rule exactly in the equal-confidence case and adds a provable \textbf{heterogeneity bonus} when the selected tokens have uneven confidences.
Fast-dLLM++ leaves the model, diffusion process, and cache implementation entirely unchanged, making it a drop-in replacement for existing Fast-dLLM decoding.
Experiments on GSM8K, MATH, HumanEval, and MBPP with the LLaDA-8B model show that the theoretical improvement translates directly into empirical gains: profile-aware selection improves the accuracy--throughput frontier by exploiting safe parallelism that weakest-token rules miss, achieving up to 37\% higher throughput at comparable accuracy. Our code is available at \url{https://github.com/Ringo-Star/FastdLLM_plusplus}.
\end{abstract}

\section{Introduction}
\label{sec:intro}

Masked diffusion language models (MDLMs) generate text by iteratively unmasking tokens from a fully masked sequence ~\citep{zhong2026parallelism,li2026diffusion, sahoo2024simple,shi2024simplified}. Unlike autoregressive models, they can in principle decode multiple positions simultaneously, offering a path to substantially faster inference. However, parallel decoding introduces a \emph{curse of parallelism}~\citep{israel2025accelerating, liu2024discrete, azangulov2025parallel}: if several masked positions are decoded independently from their marginal distributions, the resulting combination may be incoherent even when each individual token appears likely. 

Fast-dLLM \citep{wu2025fastdllm} addressed this challenge through confidence-guided parallel decoding rules, \textit{threshold and factor}, that decide how many tokens to commit per step. The factor rule accepts a candidate set of size $n$ when $(n+1)(1-c_{(n)}) < f$, where $c_{(n)}$ is the \emph{weakest} confidence among the top-$n$ candidates and $f$ is a tunable parameter. This criterion is grounded in a worst-case analysis that assumes all selected tokens share the same confidence level. This homogeneous-confidence assumption is conservative. In practice, the confidence profile across masked positions is highly heterogeneous: a few positions may have near-certain predictions while others are moderately confident. By compressing this profile to a single scalar i.e. the weakest confidence, the factor rule discards information that could safely license additional parallelism.

We propose \textbf{Fast-dLLM++}, which replaces the weakest-token selector with \emph{Fr\'{e}chet profile decoding}. The key idea is to use the full sorted confidence profile $(c_{(1)}, c_{(2)}, \ldots, c_{(n)})$ rather than only $c_{(n)}$. For each candidate prefix of size $n$, we compute a Fr\'{e}chet lower bound $L_n$ on the probability that all $n$ marginal greedy tokens are jointly correct, and an upper bound $U_n$ on the probability of any competing tuple. We commit the largest prefix satisfying $L_n - U_n > \delta$, where $\delta \ge 0$ is a margin parameter.

\begin{figure*}[t]
\centering
\resizebox{\linewidth}{!}{\input{figures/fig1_tikz}}
\caption{\textbf{Fr\'{e}chet Profile Decoding exploits heterogeneous confidence profiles to commit more tokens per denoising step.}
At each masked diffusion step, the model predicts candidate tokens with confidences \(c_i\); green marks committed tokens, gray marks deferred tokens, and red marks the factor bottleneck \(c_{(n)}\). The green shaded region denotes the heterogeneity bonus \(B_n=\sum_{j<n}(c_{(j)}-c_{(n)})\), which is the extra profile information discarded by the homogeneous weakest-token factor rule.}
\label{fig:illustrative}
\end{figure*}

Our contributions are as follows:

\begin{itemize}[leftmargin=*,itemsep=2pt,topsep=2pt]
   
\item We introduce \textbf{Fr\'{e}chet profile decoding}, a training-free parallel decoding rule for diffusion LLMs that uses the full sorted confidence profile, rather than a fixed confidence threshold or only the weakest selected token.

\item We prove that the Fr\'{e}chet criterion gives a heterogeneous-confidence sufficient condition under which greedy parallel decoding matches the true joint greedy decision (Theorem~\ref{thm:frechet-greedy}). We further show that, without additional dependence information across positions, the Fr\'{e}chet bound is the strongest distribution-free conservative certificate available from marginal confidences alone.

\item We show that Fr\'{e}chet decoding reduces exactly to Fast-dLLM factor decoding in the equal-confidence case. We then derive a \emph{heterogeneity-bonus decomposition} (Proposition~\ref{thm:heterogeneity-bonus}) that explains when and why profile-aware decoding can commit more tokens per denoising step.

\item We evaluate Fast-dLLM++ on GSM8K, MATH, HumanEval, and MBPP using LLaDA-8B-Instruct and Dream-v0-Base-7B across multiple cache regimes and generation lengths. The results show that the theoretical improvement translates into empirical gains in the accuracy--throughput frontier. We further evaluate Fast-dLLM++ on multimodal reasoning tasks with LLaDA-V.

\end{itemize}
\section{Preliminary}
\label{sec:preliminary}

\paragraph{Masked Diffusion Models.}
Masked diffusion models~\cite{lou2024sedd,sahoo2024simple} are a dominant paradigm in discrete diffusion models by introducing absorbing states in the Markov process.  For later convenience, we denote $\seqx:=(x_1,\ldots,x_L)\in\mathcal{V}^L$ a clean token sequence and let $\mask$ be an absorbing mask symbol.
At a given reverse diffusion step $t$, let $E$ denote the current \emph{evidence}: the prompt, all previously unmasked tokens, and the current partially masked sequence. Let $\mathcal{M}(E)$ be the set of masked positions.
\begin{equation}
q_t(\seqx_t\mid \seqx_0) =\prod_{i=1}^{L}\left[(1-\gamma(t))\,\mathbf{1}\{x_{t,i}=x_{0,i}\} +\gamma(t)\,\mathbf{1}\{x_{t,i}=\mask\}\right].
\label{eq:masked_forward_process}
\end{equation}
In the continuous-time formulation, a noise level $t\in[0,1]$ is associated with a masking rate schedule $\gamma(t)$, where $\gamma(0)=0$ and $\gamma(1)=1$. The learned reverse process predicts the clean token at each masked position, i.e., $p_\theta(x_{0,i}\mid \seqx_t,t),\quad i\in\mathcal{M}_t:=\{i:x_{t,i}=\mask\}.$

Earlier discrete diffusion work introduced structured categorical corruption kernels and absorbing-state variants~\cite{austin2021structured}, while later masked diffusion formulations simplified the objective into a weighted masked-token prediction loss and showed that such models can scale competitively for language modeling~\cite{sahoo2024simple,shi2024simplified,nie2025llada,ye2025dream}.
\begin{equation}
\mathcal{L}_{\mathrm{MDLM}}(\theta)
=\E_{\seqx_0,t,\seqx_t\sim q_t}\left[
 w(t)\sum_{i\in\mathcal{M}_t}-\log p_\theta(x_{0,i}\mid \seqx_t,t)
\right],
\label{eq:masked_dllm_loss}
\end{equation}
where $w(t)$ is determined by the chosen variational, score-entropy, or simplified masked-diffusion objective. During inference, the model computes marginal distributions at denoising step $k$ over masked positions and commits a subset $S_k\subseteq \mathcal{M}_k$:$\hat{x}_{i}=\arg\max_{v\in\mathcal{V}}p_\theta(v\mid \seqx_k,t_k),  i\in S_k.$

\paragraph{The Curse of Parallelism in dLLM Decoding.}
The main speed advantage of diffusion language models comes from parallelism (increasing $|S_k|$) at inference time~\cite{yu2025dimple,ben2025accelerated}. However, the same operation creates the \emph{curse of parallelism}: increasing the number of simultaneously unmasked tokens reduces the number of denoising steps, but also increases the gap between the true conditional joint distribution and the product of token-wise marginals used by dLLMs~\cite{kang2025parallelbench}.

Specifically, let $E$ denote the current evidence, and let $I=\{i_1,\ldots,i_n\}$ be a set of masked positions proposed for simultaneous decoding. The ideal conditional distribution is
\begin{equation}
p(\mathbf{z}\mid E)=p(X_{i_1}=z_1,\ldots,X_{i_n}=z_n\mid E).
\end{equation}
A standard parallel decoder instead uses the product of marginal predictions,
\begin{equation}
q(\mathbf{z}\mid E)=\prod_{j=1}^{n}p_j(z_j\mid E),
\quad
p_j(z_j\mid E):=p(X_{i_j}=z_j\mid E).
\label{eq:product_marginal_parallel}
\end{equation}
This conditional independence approximation ignores dependencies among the tokens in $I$. Consequently, decoding many positions in parallel can finalize locally plausible tokens that are globally inconsistent. This effect has been identified as a root cause of quality degradation in parallel dLLM decoding \cite{wu2025fastdllm,bansal2025enabling}.

\paragraph{Confidence-Aware Parallel Decoding.}
Confidence-aware decoding attempts to exploit parallelism only when the conditional-independence approximation is almost harmless. For each masked position $i$, define the model confidence $c_i=\max_{v\in\mathcal{V}}p_\theta(v\mid \seqx_k,t_k).$ A threshold-based decoder commits $S_k=\{i\in\mathcal{M}_k:c_i\ge \tau\}$ and, if $S_k=\emptyset$, commits the single highest-confidence token. Fast-dLLM~\cite{wu2025fastdllm} provides a useful theoretical justification for this rule through a high-confidence condition on product-of-marginals decoding. The theorem also gives a direct interpretation of confidence-aware decoding: If all selected tokens have confidence at least $1-\epsilon$, then committing $n$ tokens is safe only when $\epsilon$ is small relative to $n$.

However, the main theorem in Fast-dLLM~\cite{wu2025fastdllm} only makes use of the smallest marginal confidence $\epsilon$, which is a homogeneous case of exploiting dLLM's inference confidence profile. Therefore, our core insight is that a strictly better bound can be developed by extending the preceding theorem to heterogeneous case.
\section{Fr\'{e}chet Profile Decoding}
\label{sec:frechet}

\paragraph{Notation.}
For each $i \in \mathcal{M}(E)$, the diffusion language model induces a marginal predictive distribution $p_i(v \mid E) := p_\theta(X_i = v \mid E)$, $v \in \mathcal{V}$. Define the \emph{marginal greedy token} $x_i^\star := \argmax_{v \in \mathcal{V}}\, p_i(v \mid E)$ and its \emph{confidence} $c_i := p_i(x_i^\star \mid E)$.
For a candidate set $S = \{i_1, \ldots, i_n\} \subseteq \mathcal{M}(E)$, let $x_S^\star := (x_{i_1}^\star, \ldots, x_{i_n}^\star)$ be the tuple of marginal greedy tokens, with confidences sorted as $c_{(1)} \ge c_{(2)} \ge \cdots \ge c_{(n)}$. Let $P_S(z \mid E) := p_\theta(X_S = z \mid E)$ denote the true joint conditional distribution over the selected positions, and let $Q_S(z \mid E) := \prod_{i \in S} p_i(z_i \mid E)$ denote the product-of-marginals approximation used by parallel decoding.

\paragraph{Fr\'{e}chet profile quantities.}
For a candidate prefix of size $n$, define the \emph{Fr\'{e}chet lower bound} on joint correctness $L_n := \max\!\left\{ 0,\, \sum_{j=1}^{n} c_{(j)} - (n-1) \right\}$, the \emph{competing-event upper bound} $U_n := 1 - c_{(n)}$, and the \emph{Fr\'{e}chet score} $G_n := L_n - U_n$. Given a margin $\delta \ge 0$, Fr\'{e}chet profile decoding selects $n^\star := \max\{ n : G_n > \delta \}$, with the convention $n^\star = 1$ if the set is empty (ensuring progress).

\subsection{From weakest-token to profile-aware decoding}
\label{sec:from-weakest}

Fast-dLLM factor decoding selects a parallel commit set using only the weakest confidence $c_{(n)}$ among the top-$n$ candidate tokens. This is conservative when confidence profiles are heterogeneous: a single moderately confident token ,can prevent committing several nearly certain tokens.

We instead use the full sorted confidence profile. For the top-$n$ candidates, the Fr\'{e}chet--Hoeffding lower bound gives $L_n$ as a worst-case lower bound on the probability that all selected marginal argmax tokens are jointly correct. Any alternative tuple must differ in at least one coordinate and therefore has probability at most $U_n = 1 - c_{(n)}$. We commit the largest prefix satisfying $L_n - U_n > \delta$. This rule is training-free, requires only sorting and prefix sums, and is compatible with the cache mechanisms of Fast-dLLM. Algorithm~\ref{alg:frechet-profile-decoding} presents the complete procedure. The only change relative to Fast-dLLM is the token-selection rule in lines~11--18; the model, diffusion schedule, and caching are untouched.

\begin{wrapfigure}{r}{0.6\linewidth}
    \centering
    \begin{minipage}{\linewidth}
    \vspace{-2.5em}
\begin{algorithm}[H]
\caption{Fast-dLLM++: Fr\'{e}chet Profile Decoding}
\label{alg:frechet-profile-decoding}
\begin{algorithmic}[1]
\REQUIRE Diffusion LM $p_\theta$, prompt $p_0$, generation length $L$, block
size $B$, margin $\delta \ge 0$, cache mode
$\mathsf{cache} \in \{\textsc{None}, \textsc{Prefix}, \textsc{Dual}\}$
\STATE Initialize sequence $x \gets [p_0;\; \mask, \ldots, \mask]$
\FOR{each generation block $\mathcal{B}$}
    \STATE Initialize or refresh cache according to $\mathsf{cache}$
    \WHILE{$\mathcal{B}$ contains masked positions}
        \STATE Run $p_\theta$ on the active block/context to obtain logits
        \FOR{each masked position $i \in \mathcal{M}$}
            \STATE $\hat{x}_i \gets \argmax_{v \in \mathcal{V}}\, p_\theta(X_i = v \mid x)$
            \STATE $c_i \gets \max_{v \in \mathcal{V}}\, p_\theta(X_i = v \mid x)$
        \ENDFOR
        \STATE Sort confidences: $c_{(1)} \ge c_{(2)} \ge \cdots \ge c_{(m)}$
        \FOR{$n = 1, \ldots, m$}
            \STATE $L_n \gets \max\!\left\{0,\;\sum_{j=1}^n c_{(j)} - (n-1)\right\}$
            \STATE $U_n \gets 1 - c_{(n)}$
            \STATE $G_n \gets L_n - U_n$
        \ENDFOR
        \STATE Select $n^\star \gets \max\{n : G_n > \delta\}$
        \IF{no such $n^\star$ exists}
            \STATE $n^\star \gets 1$ \COMMENT{ensure progress}
        \ENDIF
        \STATE Reveal the $n^\star$ masked positions with largest confidence:
        $x_i \gets \hat{x}_i$
    \ENDWHILE
\ENDFOR
\RETURN $x$
\end{algorithmic}
\end{algorithm}
    \end{minipage}
    \vspace{-3.5em}
\end{wrapfigure}
\section{Theoretical Analysis}
\label{sec:theory}

The theory is organized around the amount of information available to the decoder.

\paragraph{Proof organization.}
For readability, the main text states each result together with its intuition and practical implication. All formal proofs for Sections~\ref{sec:theory} and their certification extensions in ~\ref{app:certified-frontier} are collected in Appendices~\ref{app:proofs} and~\ref{app:proofs-certified}, respectively.

\paragraph{Marginal-only certificates.}
First, we ask what can be guaranteed using only token-level marginal confidences. This is the setting of threshold, factor, and our Fr\'{e}chet selector. Theorem~\ref{thm:frechet-greedy} shows that the Fr\'{e}chet lower bound gives a distribution-free certificate: if the lower bound on the marginal-greedy tuple exceeds an upper bound on every competitor, then the parallel greedy commit agrees with the true joint greedy decision.

\paragraph{Relationship to Fast-dLLM factor decoding.}
Second, we show that Fast-dLLM's factor rule is not a separate principle but the equal-confidence specialization of the Fr\'{e}chet certificate. When all selected tokens have the same confidence, Fr\'{e}chet exactly recovers factor decoding under the parameter mapping $f=1-\delta$.

\paragraph{Why heterogeneous profiles help.}
Third, we decompose the Fr\'{e}chet score into a weakest-token factor core plus a nonnegative heterogeneity bonus. This explains the main algorithmic advantage: factor treats the selected prefix as if all tokens were as uncertain as the weakest token, while Fr\'{e}chet gives credit to the stronger tokens in the prefix.

\paragraph{Going beyond marginal-only guarantees.}
Finally, we consider what happens when we know something about the dependence structure. The total-variation and KL stability results show that if the true joint distribution $P_S$ is close to the product-of-marginals approximation $Q_S$, then the product-mode decision is stable even beyond the conservative Fr\'{e}chet regime.

Together, these results separate two regimes: Fr\'{e}chet is the sharp marginal-only certificate for heterogeneous confidence profiles, while dependence-aware information can certify additional parallel commits when the true joint is close to the product approximation.

\subsection{Heterogeneous-Confidence Greedy Equivalence}
\label{sec:greedy_equiv}

\begin{theorem}[Heterogeneous-confidence greedy equivalence]
\label{thm:frechet-greedy}
Fix evidence $E$ and a candidate set $S=\{i_1,\ldots,i_n\}$. Let $x_S^\star$ be the tuple of marginal greedy tokens and let $c_{(1)}\ge\cdots\ge c_{(n)}$ be the sorted marginal confidences. Define
\begin{equation}
L_n = \max\!\left\{0,\,\sum_{j=1}^{n}c_{(j)}-(n-1)\right\},
\qquad
U_n = 1-c_{(n)}.
\label{eq:frechet_Ln_Un}
\end{equation}
If $L_n > U_n$, then $x_S^\star$ is the unique maximizer of the true joint conditional distribution:
\begin{equation}
x_S^\star = \argmax_{z\in\mathcal{V}^n}\,P_S(z\mid E).
\label{eq:frechet_greedy_equiv}
\end{equation}
\end{theorem}
\paragraph{Intuition.}
The selected tuple $x_S^\star$ is the one obtained by taking the marginal argmax at every selected position. The Fr\'{e}chet lower bound $L_n$ is a worst-case lower bound on the probability that all these marginal choices are jointly correct. Any competing tuple must disagree in at least one coordinate, so its probability is at most the probability that the weakest selected token is wrong, $U_n=1-c_{(n)}$. If the lower bound on the selected tuple exceeds the upper bound on every competitor, the selected tuple must be the unique joint mode. Next, we show under equal confidence this reduces to exact factor decoding.

\begin{corollary}[Equal-confidence reduction]
\label{cor:factor-reduction}
Suppose $c_{(1)}=\cdots=c_{(n)}=c$. Then the Fr\'{e}chet margin criterion $G_n:=L_n-U_n>\delta$ reduces to
\begin{equation}
c > 1-\frac{1-\delta}{n+1}.
\label{eq:equal_conf_threshold}
\end{equation}
Fast-dLLM factor decoding with parameter $f$ accepts $n$ when $(n+1)(1-c_{(n)})<f$, which under equal confidences gives $c>1-f/(n+1)$. Setting $f=1-\delta$ makes the two criteria identical:
\begin{equation}
\text{Fr\'{e}chet}(\delta) \;\equiv\; \text{Factor}(f=1-\delta)
\qquad\text{(equal-confidence case)}.
\label{eq:frechet_factor_equiv}
\end{equation}
For the parameter matching $f=1-\delta$, assume $\delta\in[0,1]$ so that the matched factor is nonnegative.
\end{corollary}

\paragraph{Intuition.}
Factor decoding can be viewed as replacing the whole selected confidence profile by a flat surrogate in which every selected token has confidence $c_{(n)}$. When the real profile is actually flat, this surrogate is exact, and Fr\'{e}chet reduces to factor. The difference between the two only appears when the confidence profile is heterogeneous. The following proposition formalizes the same.

\begin{proposition}[Heterogeneity bonus]
\label{thm:heterogeneity-bonus}
Assume $L_n>0$. The Fr\'{e}chet score $G_n:=L_n-U_n$ decomposes as
\begin{equation}
G_n = \underbrace{(n+1)\,c_{(n)}-n\vphantom{\sum_{j}}}_{F_n\;\text{(factor core)}}
\;+\;\underbrace{\sum_{j=1}^{n-1}\!\bigl(c_{(j)}-c_{(n)}\bigr)}_{B_n\;\text{(heterogeneity bonus)}},
\label{eq:heterogeneity_decomp}
\end{equation}
with $B_n\ge 0$. Equality $B_n=0$ holds iff the confidence profile is flat.
\end{proposition}

Eq.~\ref{eq:heterogeneity_decomp} is the central insight: the factor core $F_n$ depends only on the weakest confidence $c_{(n)}$, while $B_n$ captures the additional information from stronger tokens. When the profile is heterogeneous, $B_n>0$ and Fr\'{e}chet decoding can accept larger commit sets than factor decoding at the same margin.

\paragraph{Intuition.}
The first term $F_n=(n+1)c_{(n)}-n$ is exactly the weakest-token logic used by factor. The second term $B_n$ is the extra area between the confidence profile and the weakest selected confidence. This area is nonnegative and measures how much information factor discards by compressing the profile to $c_{(n)}$.

\begin{corollary}[Matched-factor dominance]
\label{cor:matched-factor-dominance}
Compare Fr\'{e}chet decoding with margin $\delta$ to factor decoding with matched parameter $f=1-\delta$. If factor accepts a prefix of size $n$, then Fr\'{e}chet also accepts that prefix. Moreover, Fr\'{e}chet strictly accepts a prefix rejected by matched factor whenever
\[
F_n \le \delta < F_n+B_n,
\]
where $F_n=(n+1)c_{(n)}-n$ and $B_n=\sum_{j=1}^{n-1}(c_{(j)}-c_{(n)})$.
\end{corollary}

\paragraph{Practical takeaway.}
At matched aggressiveness $f=1-\delta$, factor acceptance implies Fr\'{e}chet acceptance. Strict separation occurs only when the heterogeneity bonus is large enough to cross the decision boundary, i.e.\ $F_n\le \delta < F_n+B_n$. Thus the benefit of Fr\'{e}chet is not automatic for every non-flat profile; it appears when the bonus is large enough to change the accept/reject decision.

\begin{corollary}[Threshold decoding under the Fr\'{e}chet criterion]
\label{cor:threshold-certificate}
Theorem~\ref{thm:frechet-greedy} certifies any proposed commit set, including one produced by threshold decoding. For $S_\tau=\{i: c_i\ge\tau\}$ with $n_\tau=|S_\tau|$, every selected confidence is at least $\tau$, so $G_{n_\tau}\ge\max\{0,\,n_\tau\tau-(n_\tau-1)\}-(1-\tau)$, i.e.\ on the positive branch a threshold set is Fr\'{e}chet-certified whenever $(n_\tau+1)\tau-n_\tau>\delta$.
\end{corollary}
This also exposes threshold decoding's structural gap: a fixed per-token gate neither adapts its required confidence to the commit-set size nor uses the profile above the weakest token. When this coarse, $\tau$-only corollary fails, the actual selected profile can still be evaluated by the general test.

\begin{remark}[Adaptive-factor interpretation]
\label{cor:adaptive-factor}
Fr\'{e}chet decoding with margin $\delta$ is equivalent to factor decoding with a data-dependent effective factor
\begin{equation}
f_{\mathrm{eff}}(n) = 1-\delta+B_n,
\label{eq:adaptive_factor}
\end{equation}
whose aggressiveness increases when the confidence profile is heterogeneous.
\end{remark}

This provides the clearest intuition: threshold decoding uses a fixed cutoff; factor decoding is set-size-aware; Fr\'{e}chet decoding is \emph{profile-aware} and \emph{data-adaptive}, automatically becoming more aggressive when the evidence supports it.

\subsection{Dependence-Aware Extensions}
\label{sec:dep_aware}

Fr\'{e}chet is intentionally marginal-only. The following stability result is not used by Fast-dLLM++, but clarifies how stronger guarantees would become possible if additional dependence information were available.

Theorem~\ref{thm:frechet-greedy} is exact over the worst case of all joint distributions consistent with the marginals (the Fr\'{e}chet class). When additional information about the dependence structure is available, sharper guarantees are possible.

Here $\Delta_Q$ is the mode gap of the product approximation: it measures how much more probability $Q_S$ assigns to the product-mode tuple $x_S^\star$ than to the best competing tuple. A large $\Delta_Q$ means the product approximation has a clear winner.

\begin{lemma}[Mode stability under total variation]
\label{thm:tv-stability}
Let $Q_S(z\mid E)=\prod_{i\in S}p_i(z_i\mid E)$ be the product-of-marginals approximation with mode gap
\begin{equation}
\Delta_Q := Q_S(x_S^\star\mid E) - \max_{z\neq x_S^\star}Q_S(z\mid E).
\label{eq:mode_gap}
\end{equation}
Assume $\Delta_Q>0$, so that $x_S^\star$ is the unique mode of $Q_S$.
If $d_{\TV}(P_S,Q_S)<\Delta_Q/2$, then $x_S^\star$ is also the unique maximizer of $P_S(\cdot\mid E)$.
\end{lemma}

\begin{corollary}[Mode stability under KL divergence]
\label{cor:kl-stability}
If $D_{\KL}(P_S\|Q_S)<\Delta_Q^2/2$, then $x_S^\star$ is the unique maximizer of $P_S(\cdot\mid E)$.
\end{corollary}

\paragraph{Intuition.}
Total variation bounds how much probability mass can change for any event when moving from $Q_S$ to $P_S$. The product-mode tuple can lose at most $d_{\TV}(P_S,Q_S)$ probability, while any competitor can gain at most the same amount. Therefore the mode gap can shrink by at most $2d_{\TV}(P_S,Q_S)$. If this is smaller than $\Delta_Q$, the winner under $Q_S$ cannot be overtaken under $P_S$.

Note that $D_{\KL}(P_S\|Q_S)$ is exactly the conditional total correlation of the token set. A small total correlation therefore guarantees that the factorized greedy choice is stable, providing a clean separation between marginal confidence and residual dependence. This connects our Fr\'{e}chet criterion (which is marginal-only and worst-case) to a richer dependence-aware regime: when the model's joint is close to its product of marginals, even tokens that fail the Fr\'{e}chet test may be safe to commit.

\paragraph{Extensions.}
The basic Fr\'{e}chet selector assumes reliable marginal confidences. Appendix~\ref{app:certified-frontier} gives a calibration-robust variant of the Fr\'{e}chet certificate, which accounts for possible model overconfidence by replacing reported confidences with conservative lower bounds.

\section{Experiments}
\label{sec:experiments}

We evaluate Fr\'{e}chet profile decoding as a drop-in replacement for Fast-dLLM's threshold and factor rules on four benchmarks, three cache regimes, and multiple generation lengths. 

\paragraph{Model and Hardware.}
We use LLaDA-8B-Instruct, a masked diffusion language model with 8B parameters. All experiments use greedy decoding (argmax per position) with no temperature sampling. All same-hardware comparisons use identical GPU configurations. We report results on a single NVIDIA H100 (80\,GB) GPU.

\paragraph{Benchmarks.}
We evaluate on GSM8K~\cite{cobbe2021gsm8k} (5-shot, 8-shot for 1024-length), MATH~\cite{hendrycks2021math} (4-shot), HumanEval~\cite{chen2021humaneval} (0-shot), and MBPP~\cite{austin2021mbpp} (3-shot). For GSM8K we report flexible-extract accuracy; for MATH we report \texttt{math\_verify} accuracy; for HumanEval we report pass@1 after code postprocessing; for MBPP we report pass@1.

\paragraph{Cache regimes.}
We test three configurations: (i)~\textbf{no cache} (nocache), where each denoising step processes the full sequence; (ii)~\textbf{prefix cache} (pcache), which caches the prompt KV states; and (iii)~\textbf{DualCache} (dcache), which additionally caches previously decoded block states.

\paragraph{Baselines.}
We compare against Fast-dLLM's two decoding rules: \emph{threshold} (accept tokens above a fixed confidence $\tau$) and \emph{factor} (accept $n$ tokens when $(n+1)(1-c_{(n)}) < f$). For fair comparison, we use matched parameters: Fr\'{e}chet margin $\delta$ is compared against factor $f = 1 - \delta$.
 
\paragraph{Parameter selection and sensitivity.}
Following Fast-dLLM, we use a single global operating point for each selector in the main tables: threshold $\tau=0.9$, factor $f=1.0$, and Fr\'{e}chet margin $\delta=0.25$, unless otherwise stated. These values are chosen from development sweeps as robust accuracy--throughput operating points and are kept fixed across benchmarks, cache regimes, and generation lengths.
Figure~\ref{fig:pareto-frontier} reports the full sweep on GSM8K and Appendix~\ref{app:matched-sweeps} the sweeps on the other benchmarks, cache regimes, and backbones. $\delta$ is an explicit speed--quality control with smooth, monotone efficiency behavior. Our recommended default is the conservative $\delta=0.25$ used throughout; when a development split exists, choose the fastest $\delta$ whose validation quality is within a prespecified tolerance. This tuning dimension is not unique to our method: threshold's $\tau$ and factor's $f$ are tuned in the same manner.

\paragraph{Memory and implementation overhead.}
Fast-dLLM++ changes only the token-selection rule. It adds no model parameters, no additional KV cache, and no persistent hidden-state storage beyond the underlying Fast-dLLM cache mode. Relative to factor decoding, it operates on the same per-position confidence vector and requires only sorting and prefix sums over the active block. Thus the method has no additional persistent memory overhead and negligible compute overhead compared with a denoising forward pass.

\paragraph{Main results.} Table~\ref{tab:main-results} summarizes the main comparisons across four benchmarks at generation lengths 256 and 512. Threshold decoding ($\tau=0.9$) is the primary decoding rule proposed in Fast-dLLM; we report improvements relative to this state-of-the-art baseline. More examples can be found in Appendix~\ref{app:text-generation-qualitative}.

\paragraph{Efficiency gains.}
Fast-dLLM's threshold decoding ($\tau=0.9$) is the state-of-the-art parallel decoding rule for diffusion LLMs. Across all eight dataset--length settings, Fr\'{e}chet profile decoding improves throughput by \textbf{1.36$\times$} on average over threshold decoding while reducing total NFE by \textbf{29.2\%}, with only a 0.48-point average accuracy change. Relative to the LLaDA-8B full-step baseline (no early stopping), Fr\'{e}chet achieves \textbf{4.31$\times$} average throughput and \textbf{79.1\%} NFE reduction, demonstrating that profile-aware selection extracts substantially more safe parallelism than both the full-step baseline and the existing threshold rule.
Figure~\ref{fig:pareto-frontier} visualizes the full accuracy--throughput frontier on GSM8K by sweeping the selector parameter for each method: Fr\'{e}chet margin $\delta\in[0,0.30]$, matched factor $f=1-\delta$, and threshold $\tau\in[0.5,0.9]$. Fr\'{e}chet shifts the matched-factor frontier to the right: at matched aggressiveness $f=1-\delta$, it consistently achieves higher throughput, with the accuracy gap shrinking in the conservative regime. We therefore interpret Fr\'{e}chet as improving the throughput--accuracy trade-off rather than strictly dominating every baseline point. This confirms that the heterogeneity bonus $B_n$ (Proposition~\ref{thm:heterogeneity-bonus}) translates into a measurable throughput advantage across the entire operating range, not just at a single parameter setting.

\paragraph{Comparison with factor at matched aggressiveness.}
Table~\ref{tab:main-results} compares Fr\'{e}chet $\delta=0.25$ with factor at its most aggressive operating point $f=1.0$, and Fr\'{e}chet still uses fewer NFE at higher throughput and comparable accuracy. The matched factor $f=1-\delta=0.75$ (Corollary~\ref{cor:factor-reduction}) is more conservative and slower; across MATH, HumanEval, MBPP, both cache regimes, and Fast-dLLM-v2 (Appendix~\ref{app:matched-sweeps}), matched Fr\'{e}chet holds accuracy within sampling noise of factor$(f=1-\delta)$ while cutting 9--20\% of forward evaluations, a consistent 10--25\% throughput gain. On the LLaDA-8B PrefixCache setting, Fr\'{e}chet commits more than factor on 31.7\% (MBPP) to 41.7\% (GSM8K) of steps, and each such step removes future forward passes; these per-step savings accumulate over the full generation into the NFE reductions in Table~\ref{tab:main-results}.

\paragraph{Why can accuracy be preserved despite greater speed?}
Fr\'{e}chet does not merely commit more tokens; it commits more tokens only when the full confidence profile provides a stronger joint-correctness certificate. In settings where the model's confidences are calibrated and the selected tokens are weakly dependent, this can reduce NFE without increasing commit errors. When these assumptions fail, especially under overconfidence or strong syntactic coupling, the method can overcommit; this is precisely the failure mode addressed by the robust and verifier-calibrated analysis in Appendix~\ref{app:certified-frontier}.

\subsection{Analysis}

\begin{figure}
    \centering
    \includegraphics[width=0.6\linewidth]{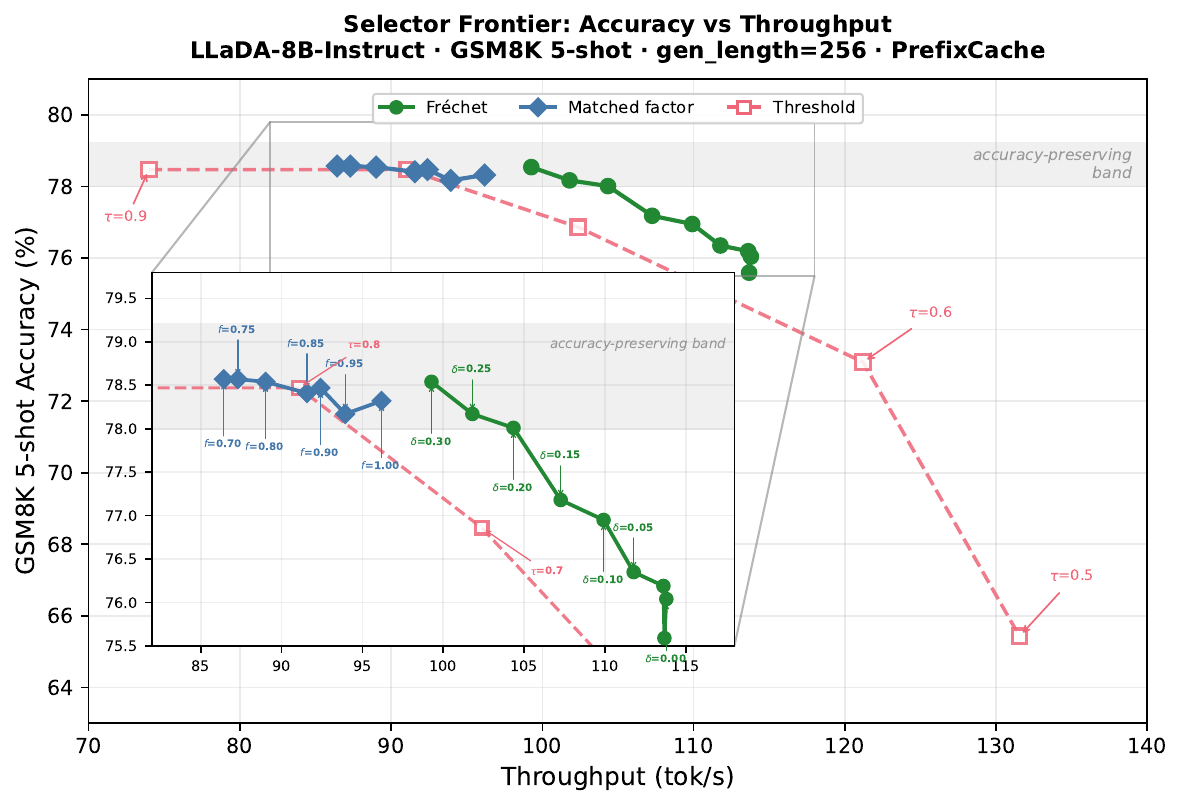}
    \caption{
    \textbf{Accuracy--throughput frontier on GSM8K.} Fr\'{e}chet shifts the matched-factor frontier toward higher throughput, especially in the conservative regime.}
    \label{fig:pareto-frontier}
\end{figure}

\paragraph{The heterogeneity bonus in practice.}
The theoretical heterogeneity bonus $B_n$ (Proposition~\ref{thm:heterogeneity-bonus}) predicts that Fr\'{e}chet decoding gains the most when the confidence profile is uneven. We observe this empirically: the largest throughput improvements occur on
GSM8K and HumanEval under prefix cache, where the model produces a mix of high-confidence function words and moderate-confidence reasoning tokens---exactly the heterogeneous regime where $B_n$ is large.

\begin{figure}[t]
\centering

\begin{minipage}[t]{0.57\textwidth}
\vspace{0pt}
\centering
\captionof{table}{
{\bf Benchmark results on LLaDA-8B-Instruct.} Our results are obtained from 
a single H100, using PrefixCache, block size $=32$.
}
\label{tab:main-results}

\begin{adjustbox}{width=\linewidth,center}
\tiny
\setlength{\tabcolsep}{1pt}
\begin{tabular}{@{}llccccc@{}}
\toprule
Dataset & Len. & Method & Acc.\ (\%) & Tok/s $\uparrow$ & NFE $\downarrow$ & Tok/NFE \\
\midrule
\multirow{3}{*}{GSM8K (5-shot)}
& 256 & Threshold & 77.6 & 73.8 (\speedup{1.00}) & 107{,}135 & 2.88 \\
& 256 & Factor & 78.1 & 96.0 (\speedup{1.30}) & 79{,}047 (\nfered{26.2}) & 3.90 \\
& 256 & Fr\'{e}chet & 77.2 & \textbf{103.8} (\speedup{1.41}) & \textbf{72{,}881} (\nfered{32.0}) & \textbf{4.24} \\
\midrule
\multirow{3}{*}{MATH (4-shot)}
& 256 & Threshold & 33.1 & 74.4 (\speedup{1.00}) & 503{,}377 & 2.48 \\
& 256 & Factor & 32.7 & 97.3 (\speedup{1.31}) & 379{,}330 (\nfered{24.6}) & 3.29 \\
& 256 & Fr\'{e}chet & 32.5 & \textbf{102.5} (\speedup{1.38}) & \textbf{358{,}178} (\nfered{28.8}) & \textbf{3.48} \\
\midrule
\multirow{3}{*}{HumanEval (0-shot)}
& 256 & Threshold & 40.2 & 77.8 (\speedup{1.00}) & 13{,}666 & 2.87 \\
& 256 & Factor & 40.7 & 89.6 (\speedup{1.15}) & 10{,}538 (\nfered{22.9}) & 3.78 \\
& 256 & Fr\'{e}chet & 40.9 & \textbf{107.7} (\speedup{1.38}) & \textbf{9{,}740} (\nfered{28.7}) & \textbf{4.06} \\
\midrule
\multirow{3}{*}{MBPP (3-shot)}
& 256 & Threshold & 27.4 & 66.0 (\speedup{1.00}) & 34{,}975 & 2.41 \\
& 256 & Factor & 21.4 & 81.6 (\speedup{1.24}) & 26{,}870 (\nfered{23.2}) & 3.16 \\
& 256 & Fr\'{e}chet & 25.4 & \textbf{85.4} (\speedup{1.29}) & \textbf{25{,}791} (\nfered{26.3}) & \textbf{3.34} \\
\midrule
\multirow{3}{*}{GSM8K (5-shot)}
& 512 & Threshold & 76.5 & 45.3 (\speedup{1.00}) & 129{,}791 & 2.75 \\
& 512 & Factor & 74.6 & 53.3 (\speedup{1.18}) & 99{,}580 (\nfered{23.3}) & 3.61 \\
& 512 & Fr\'{e}chet & 75.6 & \textbf{59.4} (\speedup{1.31}) & \textbf{91{,}239} (\nfered{29.7}) & \textbf{3.90} \\
\midrule
\multirow{3}{*}{MATH (4-shot)}
& 512 & Threshold & 36.1 & 56.4 (\speedup{1.00}) & 764{,}793 & 2.83 \\
& 512 & Factor & 35.3 & 69.1 (\speedup{1.23}) & 585{,}176 (\nfered{23.5}) & 3.70 \\
& 512 & Fr\'{e}chet & 35.5 & \textbf{77.7} (\speedup{1.38}) & \textbf{545{,}993} (\nfered{28.6}) & \textbf{3.96} \\
\midrule
\multirow{3}{*}{HumanEval (0-shot)}
& 512 & Threshold & 41.5 & 54.1 (\speedup{1.00}) & 27{,}366 & 2.80 \\
& 512 & Factor & 41.7 & 70.3 (\speedup{1.30}) & 20{,}463 (\nfered{25.2}) & 3.75 \\
& 512 & Fr\'{e}chet & 41.5 & \textbf{75.5} (\speedup{1.40}) & \textbf{18{,}909} (\nfered{30.9}) & \textbf{4.05} \\
\midrule
\multirow{3}{*}{MBPP (3-shot)}
& 512 & Threshold & 14.2 & 60.8 (\speedup{1.00}) & 59{,}999 & 2.50 \\
& 512 & Factor & 12.0 & 80.0 (\speedup{1.32}) & 45{,}186 (\nfered{24.7}) & 3.33 \\
& 512 & Fr\'{e}chet & 14.2 & \textbf{82.7} (\speedup{1.36}) & \textbf{42{,}893} (\nfered{28.5}) & \textbf{3.49} \\
\bottomrule
\end{tabular}
\end{adjustbox}

\end{minipage}
\hfill
\begin{minipage}[t]{0.40\textwidth}
\vspace{0pt}
\centering

\includegraphics[width=\linewidth]{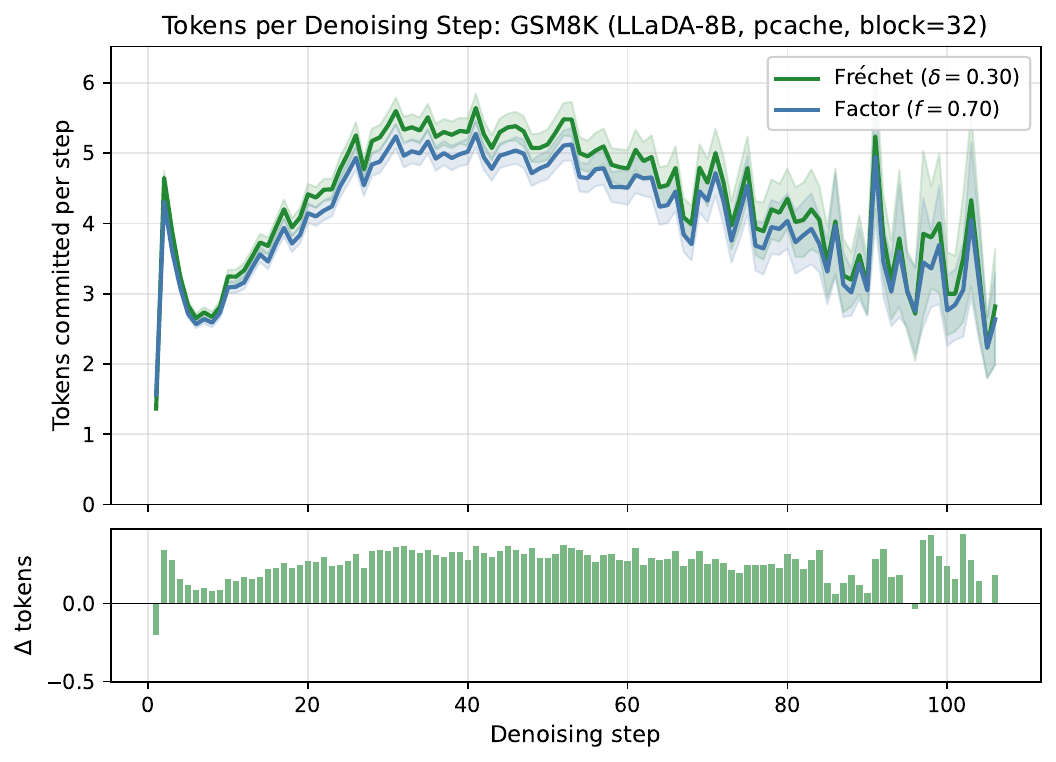}
\vspace{-1.2em}
    \captionof{figure}{\textbf{Inference dynamics of tokens committed per step.}}
\label{fig:tokens-per-step}

\vspace{0.4em}

\includegraphics[width=\linewidth]{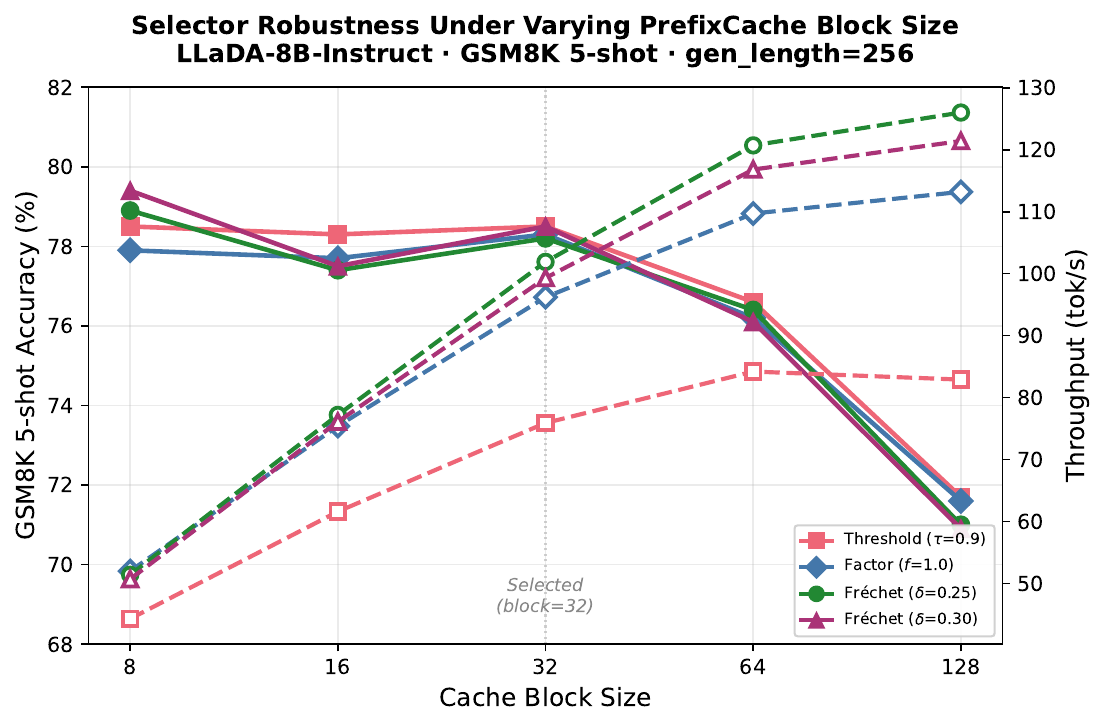}
\vspace{-1.2em}
\captionof{figure}{\textbf{Block size ablation.}}

\label{fig:block-ablation}

\end{minipage}

\end{figure}

\paragraph{Commit correctness.}
In the main setting (LLaDA-8B, PrefixCache) every parallel commit is checked against a reference: the selected positions are also decoded one token at a time from the same decoder state, and a commit is correct if it produces exactly the same tokens (Appendix~\ref{app:calibration}). The certificate is sound in practice: the observed fraction of correct commits never falls below the floor $L_n$ in any margin decile of any run. The margin is predictive: measured over factor's commits, which span the whole range of margins, $G_n$ ranks correctness with AUC .87 to .92. Fr\'{e}chet achieves the highest commit correctness of the three selectors (96.9\% on GSM8K, 97.7\% on MBPP) while using the fewest model evaluations. On MBPP, the drop to 21.4 pass@1 is driven by factor at its most aggressive operating point ($f=1.0$); matched factor ($f=0.75$, 24.0) and Fr\'{e}chet (25.3, at 11\% fewer NFE) recover most of it, and the calibration-robust variant of Appendix~\ref{app:certified-frontier} ($\eta=0.06$, 25.7) recovers it fully, cutting commit error from 4.17\% to 0.46\%. The drop is dominated by recoverable overconfidence, which affects threshold and factor at least as strongly.

\paragraph{Other backbones, tasks, and acceleration methods.}
The efficiency gain transfers to the stronger LLaDA-1.5 backbone~\citep{zhu2025llada15}, to open-ended generation (IFEval) with accuracy preserved, and to generation lengths up to 2{,}048 tokens (Appendix~\ref{app:generality}). Inserting Fr\'{e}chet at Fast-dLLM-v2's~\citep{wu2025fastdllmv2} intra-sub-block commit-selection step reduces NFE by 17--20\% at sub-block size 4 and 24--26\% at size 32 (Appendix~\ref{app:v2}). Fr\'{e}chet is also complementary to recent accelerators that alter other parts of the inference pipeline: applied on top of each (replacing only the commit rule), the gains are additive on both axes, improving both quality and efficiency (Table~\ref{tab:composition}, Appendix~\ref{app:composition}).

Figure~\ref{fig:tokens-per-step} visualizes this effect at the step level. On each denoising step, Fr\'{e}chet commits more tokens than matched factor because the heterogeneity bonus $B_n$ allows it to accept larger prefixes when the top tokens are much more confident than the weakest selected token. The extra tokens per step accumulate over the decoding trajectory, reducing total NFE by 6--8\% and producing the throughput gains reported in Table~\ref{tab:main-results}.

\paragraph{Margin sensitivity.}
The optimal margin $\delta$ varies by task: GSM8K prefers $\delta\in[0.25,0.30]$, HumanEval prefers $\delta\in[0.20,0.25]$, and MBPP is best with small margins $\delta\in[0.02,0.20]$. This is consistent with the adaptive-factor interpretation (Remark~\ref{cor:adaptive-factor}): tasks with more heterogeneous confidence profiles benefit from tighter margins that let the heterogeneity bonus do the work. Appendix~\ref{app:matched-sweeps} reports the full sweeps.

\paragraph{Generalization to Dream.}
We also evaluate on Dream-v0-Base-7B (Appendix~\ref{app:dream-results}). Fr\'{e}chet achieves 1.28$\times$ average throughput over threshold with 21.1\% NFE reduction, confirming that the profile-aware advantage transfers across diffusion LM architectures.

\paragraph{Long-context and cache-mode analysis.}
On GSM8K with generation length 1024, Fr\'{e}chet improves throughput and Tok/NFE over matched factor across all cache modes (no cache, PrefixCache, DualCache), showing that the selector gain persists in Fast-dLLM's long-generation regime. Full tables are in Appendix~\ref{app:additional-experiments}.

\paragraph{Fast-dLLM++ for complex multimodal reasoning tasks.} We test Fast-dLLM++ on multimodal diffusion large language model to test its robustness under mathematical reasoning settings. More experiment results are in Appendix~\ref{sec:exp_dmllm}.
\section{Related Work}
\label{sec:related-work}

\paragraph{Diffusion language models.}
Discrete and masked diffusion language models~\citep{austin2021structured,sahoo2024simple,shi2024simplified,lou2024sedd} offer non-autoregressive alternatives to standard left-to-right generation. LLaDA~\citep{nie2025llada} and Dream~\citep{ye2025dream} scale masked diffusion to 7--8B parameters with competitive quality. We build on these models and focus on inference-time decoding rather than training.

\paragraph{Fast diffusion decoding.}
Fast-dLLM~\citep{wu2025fastdllm} accelerates diffusion LLM inference with approximate KV caching (PrefixCache, DualCache) and confidence-aware parallel decoding via threshold and factor rules. Subsequent work explores adaptive parallel decoding~\citep{israel2025accelerating}, learned parallel schedules~\citep{bao2025learning,chen2025dparallel}, and KL-guided selection~\citep{kim2025klass}. Fast-dLLM++ is complementary to caching: it changes only the token-selection rule and composes with all cache modes. The key distinction is that factor decoding depends only on the weakest selected confidence $c_{(n)}$, while Fr\'{e}chet uses the full sorted profile.

\paragraph{Concurrent acceleration methods.}
Concurrent training-free acceleration of dLLM inference spans several mechanisms: confidence-based selectors (KLASS~\citep{kim2025klass}, CCD~\citep{ccd2026}) decide per-step commit sets from marginal confidences; dependence-aware methods (DAPD~\citep{dapd2026}, PUNT~\citep{azangulov2025parallel}) model what a marginal factorization ignores; adaptive schedulers (Dynamic-dLLM~\citep{dynamicdllm2026}, SlowFast~\citep{slowfast2026}, DUS~\citep{dus2026}) reallocate where and how many tokens are decoded; draft-and-verify decoders (WINO~\citep{wino2026}, DC-Leap~\citep{dcleap2026}) trade verification compute for aggressive drafting; cache and system methods (dLLM-Cache~\citep{liu2025dllmcache}, Streaming-dLLM~\citep{streaming2026}, DPad~\citep{chen2026dpad}) cut per-step cost; adaptive block partition (Fast-dLLM-v2~\citep{wu2025fastdllmv2}, Swordsman~\citep{swordsman2026}, DSB~\citep{dsb2026}) restructures the canvas; and distilled decoders (dParallel~\citep{chen2025dparallel}, d3LLM~\citep{d3llm2026}) internalize parallelism at training time. Fast-dLLM++ changes only the token-selection rule and is complementary to these mechanisms: applied on top of them it improves both quality and efficiency (Appendices~\ref{app:composition} and~\ref{app:related-extended}).
\section{Limitations and Conclusion}
\label{sec:limitations}

\paragraph{Limitations.}
Fr\'{e}chet certificate is the strongest distribution-free guarantee available from marginals alone, but tasks with strong syntactic or semantic coupling may require stricter margins or future dependence-aware extensions. When confidences are miscalibrated or selected tokens are strongly coupled, especially in syntactically constrained generation, profile-aware decoding may overcommit. The certificate assumes block-wise decoding; without blocks, the high-confidence end-of-sequence tokens are committed early and those positions must be excluded (Appendix~\ref{app:composition}). Distilled models have sharper confidence profiles, which reduces the heterogeneity bonus. Stronger guarantees would require calibrated confidences or explicit dependence estimates. Our analysis focuses on greedy decoding, where each position contributes only its marginal argmax token and top-1 confidence. Extending profile-aware certificates to temperature sampling, nucleus sampling, or top-\(k\) token sets remains open; one possible direction is to certify set-valued correctness events rather than equality with the joint greedy mode. The selector also ignores other distributional signals, such as entropy, top-\(k\) gaps, and the full softmax shape, which may provide useful information for commit decisions. Finally, our dependence-aware stability result (Lemma~\ref{thm:tv-stability}) suggests a path beyond marginal-only selection: when selected tokens have small residual dependence, the product-mode decision is stable. Estimating such dependence online could certify additional parallel commits and yield further speedups which is an interesting future work.

\paragraph{Conclusion.}
Fast-dLLM++ shows that selector design remains a key bottleneck in diffusion LLM inference. By replacing weakest-token selection with a Fr\'{e}chet profile certificate, it generalizes factor decoding and exploits heterogeneous confidence profiles. The resulting decoder is training-free, drop-in, and improves the throughput--accuracy frontier without modifying the model, diffusion process, or cache implementation.

\begin{ack}
This research is funded in part by an ARC (Australian Research Council) Discovery Grant of DP220100800. Prof.~Hongdong Li holds concurrent appointments as a Full Professor with the ANU and as an Amazon Scholar with Amazon (part time).
\end{ack}

\bibliography{ref}
\bibliographystyle{plainnat}

\clearpage
\appendix
\section{Full Proofs}\label{app:proofs}

\paragraph{Notation.}
Throughout this appendix, $E$ denotes the evidence event (i.e., the conditioning event representing all observed tokens and the current diffusion state). For a candidate set $S = \{i_1, \ldots, i_n\}$ of masked positions, we define the events $A_j := \{X_{i_j} = x_{i_j}^\star\}$, where $x_{i_j}^\star$ is the marginal greedy token at position $i_j$. The marginal confidence at position $i_j$ is $c_j := \Pr(A_j \mid E) = \Pr(X_{i_j} = x_{i_j}^\star \mid E)$. We write $c_{(1)} \ge c_{(2)} \ge \cdots \ge c_{(n)}$ for the sorted (order-statistic) confidences. All probabilities are conditional on $E$ unless stated otherwise; we sometimes suppress the conditioning for brevity.


\begin{lemma}[Fr\'{e}chet--Bonferroni lower bound]
\label{lem:frechet-bonferroni}
Let $A_1, \ldots, A_n$ be events in a probability space $(\Omega, \mathcal{F}, \Pr)$. Then
\begin{equation}
\Pr\!\left(\bigcap_{j=1}^n A_j\right) \ge \max\!\left\{0,\; \sum_{j=1}^n \Pr(A_j) - (n-1)\right\}.
\label{eq:frechet_bonferroni_lemma}
\end{equation}
\end{lemma}

\begin{proof}
We prove this from first principles using the complement and a union bound.

\paragraph{Step 1: Express the intersection via complements.}
By De Morgan's law,
\begin{equation}
\Pr\!\left(\bigcap_{j=1}^n A_j\right) = 1 - \Pr\!\left(\bigcup_{j=1}^n A_j^c\right).
\label{eq:demorgan_step}
\end{equation}

\paragraph{Step 2: Apply the union bound (Boole's inequality).}
The union bound states that for any events $B_1, \ldots, B_n$,
\begin{equation}
\Pr\!\left(\bigcup_{j=1}^n B_j\right) \le \sum_{j=1}^n \Pr(B_j).
\label{eq:union_bound}
\end{equation}
Applying this to $B_j = A_j^c$:
\begin{equation}
\Pr\!\left(\bigcup_{j=1}^n A_j^c\right) \le \sum_{j=1}^n \Pr(A_j^c) = \sum_{j=1}^n \bigl(1 - \Pr(A_j)\bigr) = n - \sum_{j=1}^n \Pr(A_j).
\label{eq:union_bound_applied}
\end{equation}

\paragraph{Step 3: Combine.}
Substituting Eq.~\ref{eq:union_bound_applied} into Eq.~\ref{eq:demorgan_step}:
\begin{align}
\Pr\!\left(\bigcap_{j=1}^n A_j\right)
&= 1 - \Pr\!\left(\bigcup_{j=1}^n A_j^c\right) \notag \\
&\ge 1 - \left(n - \sum_{j=1}^n \Pr(A_j)\right) \notag \\
&= \sum_{j=1}^n \Pr(A_j) - (n-1).
\label{eq:frechet_lower_derived}
\end{align}

\paragraph{Step 4: Non-negativity.}
Since probabilities are non-negative, $\Pr(\bigcap_{j=1}^n A_j) \ge 0$ always holds. Combining with Eq.~\ref{eq:frechet_lower_derived}:
\begin{equation}
\Pr\!\left(\bigcap_{j=1}^n A_j\right) \ge \max\!\left\{0,\; \sum_{j=1}^n \Pr(A_j) - (n-1)\right\}.
\end{equation}
This completes the proof.
\end{proof}

\paragraph{Citation note.}
Lemma~\ref{lem:frechet-bonferroni} is the event-probability form of the classical Fr\'{e}chet--Hoeffding / Fr\'{e}chet--Bonferroni lower bound. It also follows directly from Boole's union bound applied to the complement events. The bound is sharp as a distribution-free inequality over events with fixed marginal probabilities: when the positive branch is active, equality is obtained when the complement events are disjoint up to null sets; when the positive branch is inactive, the intersection can be made empty. In copula terminology, the bivariate lower bound is $W(u,v)=\max\{u+v-1,0\}$, while the comonotonic copula gives the corresponding upper bound $\min\{u,v\}$.


\subsection{Proof of Theorem~\ref{thm:frechet-greedy}}

\begin{proof}
We prove that if $L_n > U_n$, then the tuple of marginal greedy tokens $x_S^\star$ is the unique maximizer of the true joint conditional $P_S(\cdot \mid E)$.

Recall the definitions: $A_j := \{X_{i_j} = x_{i_j}^\star\}$ with $\Pr(A_j \mid E) = c_{(j)}$ (after sorting), and
\begin{equation}
L_n = \max\!\left\{0,\; \sum_{j=1}^n c_{(j)} - (n-1)\right\}, \qquad U_n = 1 - c_{(n)}.
\end{equation}

\paragraph{Step 1: Lower bound on the correct tuple.}
The event $\{X_S = x_S^\star\}$ is exactly the intersection $\bigcap_{j=1}^n A_j$. Applying Lemma~\ref{lem:frechet-bonferroni} to the conditional probability measure $\Pr(\cdot \mid E)$:
\begin{align}
P_S(x_S^\star \mid E)
&= \Pr\!\left(\bigcap_{j=1}^n A_j \;\middle|\; E\right) \notag \\
&\ge \max\!\left\{0,\; \sum_{j=1}^n \Pr(A_j \mid E) - (n-1)\right\} \notag \\
&= \max\!\left\{0,\; \sum_{j=1}^n c_{(j)} - (n-1)\right\} = L_n.
\label{eq:proof_frechet_lower}
\end{align}
Since we assume $L_n > U_n \ge 0$, we have $L_n > 0$, so the max is achieved by the sum expression.

\paragraph{Step 2: Upper bound on any competitor.}
Fix any competing tuple $z \in \mathcal{V}^n$ with $z \neq x_S^\star$. Since $z \neq x_S^\star$, there exists at least one coordinate $k \in \{1, \ldots, n\}$ such that $z_k \neq x_{i_k}^\star$. Now observe that:
\begin{align}
P_S(z \mid E)
&= \Pr(X_{i_1} = z_1, \ldots, X_{i_n} = z_n \mid E) \notag \\
&\le \Pr(X_{i_k} = z_k \mid E) \tag{marginalizing out other coordinates} \\
&\le 1 - \Pr(X_{i_k} = x_{i_k}^\star \mid E) \tag{since $z_k \neq x_{i_k}^\star$ and $x_{i_k}^\star$ is the greedy argmax} \\
&= 1 - c_k \notag \\
&\le 1 - c_{(n)} = U_n.
\label{eq:proof_competitor_upper}
\end{align}
The second inequality uses the fact that $x_{i_k}^\star = \argmax_v \Pr(X_{i_k} = v \mid E)$, so for any $z_k \neq x_{i_k}^\star$, we have $\Pr(X_{i_k} = z_k \mid E) \le 1 - c_k$ (since the greedy token takes probability $c_k$ and the remaining probability $1 - c_k$ is shared among all other tokens). The last inequality uses $c_k \ge c_{(n)}$ (since $c_{(n)}$ is the minimum confidence).

\paragraph{Step 3: Conclude uniqueness.}
Combining Steps 1 and 2: under the assumption $L_n > U_n$,
\begin{equation}
P_S(x_S^\star \mid E) \ge L_n > U_n \ge P_S(z \mid E)
\end{equation}
for every $z \neq x_S^\star$. Therefore $x_S^\star$ is the unique maximizer of $P_S(\cdot \mid E)$ over $\mathcal{V}^n$.
\end{proof}


\subsection{Proof of Corollary~\ref{cor:factor-reduction}}

\begin{proof}
Assume equal confidences: $c_{(1)} = c_{(2)} = \cdots = c_{(n)} = c$ for some $c \in (0,1]$.

\paragraph{Step 1: Compute $L_n$ and $U_n$.}
Under equal confidences:
\begin{align}
L_n &= \max\!\left\{0,\; \sum_{j=1}^n c - (n-1)\right\} = \max\{0,\; nc - (n-1)\} = \max\{0,\; 1 - n(1-c)\}, \label{eq:Ln_equal_proof} \\
U_n &= 1 - c. \label{eq:Un_equal_proof}
\end{align}

\paragraph{Step 2: Compute the Fr\'{e}chet margin $G_n$.}
The Fr\'{e}chet margin is defined as $G_n := L_n - U_n$. Assuming $L_n > 0$ (which is necessary for the criterion to be active):
\begin{align}
G_n &= \bigl[nc - (n-1)\bigr] - (1-c) \notag \\
    &= nc - n + 1 - 1 + c \notag \\
    &= (n+1)c - n \notag \\
    &= 1 - (n+1)(1-c). \label{eq:Gn_equal_proof}
\end{align}

\paragraph{Step 3: Derive the acceptance condition.}
The Fr\'{e}chet criterion requires $G_n > \delta$:
\begin{align}
1 - (n+1)(1-c) &> \delta \notag \\
(n+1)(1-c) &< 1 - \delta \notag \\
1 - c &< \frac{1-\delta}{n+1} \notag \\
c &> 1 - \frac{1-\delta}{n+1}. \label{eq:frechet_equal_threshold}
\end{align}

\paragraph{Step 4: Compare with factor decoding.}
Fast-dLLM factor decoding with parameter $f$ accepts a prefix of size $n$ when $(n+1)(1 - c_{(n)}) < f$. Under equal confidences $c_{(n)} = c$, this becomes:
\begin{align}
(n+1)(1-c) &< f \notag \\
c &> 1 - \frac{f}{n+1}. \label{eq:factor_equal_threshold}
\end{align}
Comparing Eq.~\ref{eq:frechet_equal_threshold} and Eq.~\ref{eq:factor_equal_threshold}: setting $f = 1 - \delta$ makes the two conditions identical. This establishes the exact equivalence $\text{Fr\'{e}chet}(\delta) \equiv \text{Factor}(f = 1-\delta)$ in the equal-confidence case.
\end{proof}


\subsection{Proof of Proposition~\ref{thm:heterogeneity-bonus}}

\begin{proof}
We decompose the Fr\'{e}chet score $G_n = L_n - U_n$ into a factor core plus a heterogeneity bonus.

\paragraph{Step 1: Expand $G_n$.}
Assuming $L_n > 0$ (stated as a hypothesis of the theorem):
\begin{align}
G_n &= L_n - U_n \notag \\
    &= \left[\sum_{j=1}^n c_{(j)} - (n-1)\right] - (1 - c_{(n)}) \notag \\
    &= \sum_{j=1}^n c_{(j)} - n + 1 - 1 + c_{(n)} \notag \\
    &= \sum_{j=1}^n c_{(j)} + c_{(n)} - n. \label{eq:Gn_expanded_proof}
\end{align}

\paragraph{Step 2: Isolate the factor core.}
Write the sum as $\sum_{j=1}^n c_{(j)} = \sum_{j=1}^{n-1} c_{(j)} + c_{(n)}$. Substituting into Eq.~\ref{eq:Gn_expanded_proof}:
\begin{align}
G_n &= \sum_{j=1}^{n-1} c_{(j)} + c_{(n)} + c_{(n)} - n \notag \\
    &= \sum_{j=1}^{n-1} c_{(j)} + 2c_{(n)} - n. \label{eq:Gn_intermediate}
\end{align}
Now add and subtract $(n-1)c_{(n)}$:
\begin{align}
G_n &= \sum_{j=1}^{n-1} c_{(j)} - (n-1)c_{(n)} + (n-1)c_{(n)} + 2c_{(n)} - n \notag \\
    &= \sum_{j=1}^{n-1}\bigl(c_{(j)} - c_{(n)}\bigr) + (n+1)c_{(n)} - n. \label{eq:Gn_decomposed_proof}
\end{align}

\paragraph{Step 3: Identify the two terms.}
Define:
\begin{align}
F_n &:= (n+1)\,c_{(n)} - n, \label{eq:Fn_def_proof} \\
B_n &:= \sum_{j=1}^{n-1}\bigl(c_{(j)} - c_{(n)}\bigr). \label{eq:Bn_def_proof}
\end{align}
Then $G_n = F_n + B_n$ by Eq.~\ref{eq:Gn_decomposed_proof}.

\paragraph{Step 4: Non-negativity of $B_n$.}
Since the confidences are sorted in non-increasing order, $c_{(j)} \ge c_{(n)}$ for all $j \in \{1, \ldots, n-1\}$. Therefore each term $c_{(j)} - c_{(n)} \ge 0$, and the sum $B_n \ge 0$.

\paragraph{Step 5: Characterize equality.}
$B_n = 0$ if and only if $c_{(j)} - c_{(n)} = 0$ for all $j = 1, \ldots, n-1$, which holds if and only if $c_{(1)} = c_{(2)} = \cdots = c_{(n)}$, i.e., the confidence profile is flat (homogeneous).
\end{proof}


\subsection{Proof of Corollary~\ref{cor:matched-factor-dominance}}

\begin{proof}
We show two claims: (i) factor acceptance implies Fr\'{e}chet acceptance, and (ii) there exist configurations where Fr\'{e}chet accepts but factor does not.

\paragraph{Step 1: Factor acceptance implies Fr\'{e}chet acceptance.}
Factor decoding with matched parameter $f = 1 - \delta$ accepts a prefix of size $n$ when
\begin{equation}
(n+1)(1 - c_{(n)}) < f = 1 - \delta,
\end{equation}
which rearranges to $(n+1)c_{(n)} - n > \delta$, i.e., $F_n > \delta$.

By Proposition~\ref{thm:heterogeneity-bonus}, $G_n = F_n + B_n$ with $B_n \ge 0$. Therefore:
\begin{equation}
F_n > \delta \implies G_n = F_n + B_n \ge F_n > \delta.
\end{equation}
So Fr\'{e}chet also accepts the prefix (since the Fr\'{e}chet criterion is $G_n > \delta$).

\paragraph{Step 2: Strict separation.}
Fr\'{e}chet strictly accepts a prefix that factor rejects when:
\begin{equation}
F_n \le \delta \quad \text{(factor rejects)} \qquad \text{and} \qquad G_n = F_n + B_n > \delta \quad \text{(Fr\'{e}chet accepts)}.
\end{equation}
These two conditions hold simultaneously if and only if $F_n \le \delta < F_n + B_n$. This requires $B_n > 0$, which by Proposition~\ref{thm:heterogeneity-bonus} holds whenever the confidence profile is not flat. In this regime, the heterogeneity bonus $B_n$ provides exactly the additional margin needed for Fr\'{e}chet acceptance.
\end{proof}


\subsection{Proof of Lemma~\ref{thm:tv-stability}}

\begin{proof}
We show that if the true joint $P_S$ is close to the product-of-marginals $Q_S$ in total variation distance, then the mode of $Q_S$ is preserved as the mode of $P_S$.

Let $Q_S(z \mid E) = \prod_{i \in S} p_i(z_i \mid E)$ be the product-of-marginals approximation, and let $\Delta_Q := Q_S(x_S^\star \mid E) - \max_{z \neq x_S^\star} Q_S(z \mid E) > 0$ be the mode gap of $Q_S$.

\paragraph{Step 1: Recall the definition of total variation distance.}
For discrete distributions $P$ and $Q$ on a finite set $\mathcal{X}$:
\begin{equation}
d_{\TV}(P, Q) = \frac{1}{2}\sum_{x \in \mathcal{X}} |P(x) - Q(x)| = \max_{A \subseteq \mathcal{X}} |P(A) - Q(A)|.
\end{equation}
In particular, for any single atom $z$: $|P_S(z \mid E) - Q_S(z \mid E)| \le d_{\TV}(P_S, Q_S)$.

\paragraph{Step 2: Lower bound on $P_S(x_S^\star \mid E)$.}
Applying the pointwise bound to the mode:
\begin{align}
P_S(x_S^\star \mid E)
&\ge Q_S(x_S^\star \mid E) - |P_S(x_S^\star \mid E) - Q_S(x_S^\star \mid E)| \notag \\
&\ge Q_S(x_S^\star \mid E) - d_{\TV}(P_S, Q_S). \label{eq:tv_lower_proof}
\end{align}

\paragraph{Step 3: Upper bound on $P_S(z \mid E)$ for any competitor $z \neq x_S^\star$.}
For any $z \neq x_S^\star$:
\begin{align}
P_S(z \mid E)
&\le Q_S(z \mid E) + |P_S(z \mid E) - Q_S(z \mid E)| \notag \\
&\le Q_S(z \mid E) + d_{\TV}(P_S, Q_S) \notag \\
&\le \bigl[Q_S(x_S^\star \mid E) - \Delta_Q\bigr] + d_{\TV}(P_S, Q_S), \label{eq:tv_upper_proof}
\end{align}
where the last inequality uses the definition of $\Delta_Q$: $Q_S(z \mid E) \le Q_S(x_S^\star \mid E) - \Delta_Q$ for all $z \neq x_S^\star$.

\paragraph{Step 4: Combine and conclude.}
Subtracting Eq.~\ref{eq:tv_upper_proof} from Eq.~\ref{eq:tv_lower_proof}:
\begin{align}
P_S(x_S^\star \mid E) - P_S(z \mid E)
&\ge \bigl[Q_S(x_S^\star \mid E) - d_{\TV}\bigr] - \bigl[Q_S(x_S^\star \mid E) - \Delta_Q + d_{\TV}\bigr] \notag \\
&= \Delta_Q - 2\,d_{\TV}(P_S, Q_S). \label{eq:tv_gap_proof}
\end{align}
If $d_{\TV}(P_S, Q_S) < \Delta_Q / 2$, then $\Delta_Q - 2\,d_{\TV} > 0$, so $P_S(x_S^\star \mid E) > P_S(z \mid E)$ for all $z \neq x_S^\star$. Hence $x_S^\star$ is the unique maximizer of $P_S(\cdot \mid E)$.
\end{proof}


\subsection{Proof of Corollary~\ref{cor:kl-stability}}

\begin{proof}
We reduce the KL condition to the TV condition of Lemma~\ref{thm:tv-stability} via Pinsker's inequality.

\paragraph{Step 1: State Pinsker's inequality.}
For any two distributions $P$ and $Q$ on the same measurable space:
\begin{equation}
d_{\TV}(P, Q) \le \sqrt{\frac{1}{2}\,D_{\KL}(P \| Q)}.
\label{eq:pinsker}
\end{equation}

\paragraph{Step 2: Apply to our setting.}
Suppose $D_{\KL}(P_S \| Q_S) < \Delta_Q^2 / 2$. Then by Pinsker's inequality:
\begin{align}
d_{\TV}(P_S, Q_S)
&\le \sqrt{\frac{1}{2}\,D_{\KL}(P_S \| Q_S)} \notag \\
&< \sqrt{\frac{1}{2} \cdot \frac{\Delta_Q^2}{2}} \notag \\
&= \sqrt{\frac{\Delta_Q^2}{4}} = \frac{\Delta_Q}{2}. \label{eq:pinsker_applied}
\end{align}

\paragraph{Step 3: Invoke Lemma~\ref{thm:tv-stability}.}
Since $d_{\TV}(P_S, Q_S) < \Delta_Q / 2$, the hypothesis of Lemma~\ref{thm:tv-stability} is satisfied. Therefore $x_S^\star$ is the unique maximizer of $P_S(\cdot \mid E)$.
\end{proof}


\section{Proofs for Certified Frontier}\label{app:proofs-certified}


\subsection{Proof of Proposition~\ref{thm:robust-frechet}}

\begin{proof}
We extend the Fr\'{e}chet greedy equivalence (Theorem~\ref{thm:frechet-greedy}) to the calibration-robust setting where true confidences $c_i$ are bounded below by $\underline{c}_i := \max\{0, \hat{c}_i - \eta_i\}$.

\paragraph{Step 1: Robust lower bound on the correct tuple.}
Define $A_i := \{X_i = x_i^\star\}$ for each $i \in S$. By hypothesis, $c_i = \Pr(A_i \mid E) \ge \underline{c}_i$ for all $i \in S$. Applying Lemma~\ref{lem:frechet-bonferroni} to the conditional measure $\Pr(\cdot \mid E)$:
\begin{align}
\Pr(X_S = x_S^\star \mid E)
&= \Pr\!\left(\bigcap_{i \in S} A_i \;\middle|\; E\right) \notag \\
&\ge \max\!\left\{0,\; \sum_{i \in S} \Pr(A_i \mid E) - (|S|-1)\right\} \notag \\
&= \max\!\left\{0,\; \sum_{i \in S} c_i - (|S|-1)\right\}. \label{eq:robust_frechet_step1}
\end{align}
Since $c_i \ge \underline{c}_i$ for all $i \in S$, and the function $\max\{0, \sum_i x_i - (|S|-1)\}$ is non-decreasing in each $x_i$:
\begin{equation}
\Pr(X_S = x_S^\star \mid E) \ge \max\!\left\{0,\; \sum_{i \in S} \underline{c}_i - (|S|-1)\right\} = L_\eta(S).
\label{eq:robust_lower_final}
\end{equation}

\paragraph{Step 2: Robust upper bound on any competitor.}
Fix any $z \neq x_S^\star$. There exists $k \in S$ with $z_k \neq x_k^\star$. Then:
\begin{align}
\Pr(X_S = z \mid E)
&\le \Pr(X_k = z_k \mid E) \tag{marginalization} \\
&\le 1 - \Pr(X_k = x_k^\star \mid E) \tag{$x_k^\star$ is the greedy argmax} \\
&= 1 - c_k \notag \\
&\le 1 - \underline{c}_k \tag{since $c_k \ge \underline{c}_k$} \\
&\le 1 - \min_{i \in S} \underline{c}_i = U_\eta(S). \label{eq:robust_upper_final}
\end{align}

\paragraph{Step 3: Conclude.}
If $L_\eta(S) > U_\eta(S)$, then combining Eqs.~\ref{eq:robust_lower_final} and~\ref{eq:robust_upper_final}:
\begin{equation}
\Pr(X_S = x_S^\star \mid E) \ge L_\eta(S) > U_\eta(S) \ge \Pr(X_S = z \mid E)
\end{equation}
for every $z \neq x_S^\star$. Therefore $x_S^\star$ is the unique maximizer of $P_S(\cdot \mid E)$.
\end{proof}


\subsection{Proof of Corollary~\ref{cor:calibration-penalty}}

\begin{proof}
We derive the calibration penalty under uniform $\eta$ and the no-clipping assumption.

\paragraph{Step 1: Compute the robust lower bound.}
Under uniform calibration error $\eta_i = \eta$ for all $i$, the robust confidences are $\underline{c}_i = \max\{0, \hat{c}_i - \eta\}$. The no-clipping assumption states $\hat{c}_{(j)} \ge \eta$ for all $j \le n$, so $\underline{c}_{(j)} = \hat{c}_{(j)} - \eta$ for all selected positions. Therefore:
\begin{align}
L_\eta
&= \max\!\left\{0,\; \sum_{j=1}^n \underline{c}_{(j)} - (n-1)\right\} \notag \\
&= \max\!\left\{0,\; \sum_{j=1}^n (\hat{c}_{(j)} - \eta) - (n-1)\right\} \notag \\
&= \max\!\left\{0,\; \sum_{j=1}^n \hat{c}_{(j)} - n\eta - (n-1)\right\}. \label{eq:cal_Ln}
\end{align}
Assuming the robust lower bound is active (i.e., the expression inside the max is positive):
\begin{equation}
L_\eta = \sum_{j=1}^n \hat{c}_{(j)} - n\eta - (n-1). \label{eq:cal_Ln_active}
\end{equation}

\paragraph{Step 2: Compute the robust upper bound.}
\begin{align}
U_\eta
&= 1 - \min_{j \le n} \underline{c}_{(j)} = 1 - \underline{c}_{(n)} \notag \\
&= 1 - (\hat{c}_{(n)} - \eta) = 1 - \hat{c}_{(n)} + \eta. \label{eq:cal_Un}
\end{align}

\paragraph{Step 3: Compute the robust Fr\'{e}chet score.}
\begin{align}
G_n^{(\eta)}
&= L_\eta - U_\eta \notag \\
&= \left[\sum_{j=1}^n \hat{c}_{(j)} - n\eta - (n-1)\right] - \left[1 - \hat{c}_{(n)} + \eta\right] \notag \\
&= \sum_{j=1}^n \hat{c}_{(j)} + \hat{c}_{(n)} - n - n\eta - \eta \notag \\
&= \left[\sum_{j=1}^n \hat{c}_{(j)} + \hat{c}_{(n)} - n\right] - (n+1)\eta. \label{eq:cal_Gn_expand}
\end{align}

\paragraph{Step 4: Identify the uncorrected score.}
The uncorrected (naive) Fr\'{e}chet score is obtained by setting $\eta = 0$:
\begin{equation}
\hat{G}_n = \sum_{j=1}^n \hat{c}_{(j)} + \hat{c}_{(n)} - n.
\end{equation}
Therefore:
\begin{equation}
G_n^{(\eta)} = \hat{G}_n - (n+1)\eta.
\end{equation}
The calibration penalty is $(n+1)\eta$, which grows linearly with both the commit set size $n$ and the calibration error $\eta$.
\end{proof}

\section{Calibration-Robust Fr\'{e}chet Certificates}
\label{app:certified-frontier}

\subsection{Reported Confidence versus Actual Correctness}
\label{sec:robust-frechet}

\paragraph{Reported confidence versus actual correctness.}
In Section~\ref{sec:theory}, we used $c_i$ for the ideal marginal confidence entering the theoretical certificate. In an implementation, however, the decoder observes a model-reported confidence
\[
\hat c_i := \max_{v\in\mathcal V}p_\theta(X_i=v\mid E).
\]
This reported value may be overconfident relative to actual correctness. To discuss calibration, let
\[
r_i := \Pr(A_i\mid E), \qquad A_i=\{X_i=x_i^\star\},
\]
denote the actual probability that the proposed token is correct under the evaluation distribution or a chosen verifier reference. We assume a one-sided calibration envelope
\[
r_i \ge \underline{c}_i := \max\{0,\hat c_i-\eta_i\},
\]
where $\eta_i\ge 0$ is a calibration allowance. Thus $\underline{c}_i$ is a conservative lower confidence: we do not trust the raw model confidence $\hat c_i$ fully, but we assume that subtracting $\eta_i$ makes it safe.

Define the robust Fr\'{e}chet lower bound and competitor upper bound:
\begin{equation}
L_\eta(S) := \max\!\left\{0,\, \sum_{i \in S} \underline{c}_i - (|S|-1)\right\},
\qquad
U_\eta(S) := 1 - \min_{i \in S} \underline{c}_i.
\label{eq:robust_L_U}
\end{equation}

\begin{proposition}[Calibration-robust Fr\'{e}chet certificate]
\label{thm:robust-frechet}
If $r_i \ge \underline{c}_i$ for all $i \in S$ and $L_\eta(S) > U_\eta(S)$, then $x_S^\star$ is the unique maximizer of $P_S(\cdot \mid E)$.
\end{proposition}

\begin{corollary}[Calibration penalty]
\label{cor:calibration-penalty}
Under uniform calibration error $\eta$, assuming no selected robust confidence is clipped at zero (i.e., $\hat{c}_{(j)} \ge \eta$ for all $j \le n$) and the robust lower bound is active, the robust Fr\'{e}chet score satisfies
\begin{equation}
G_n^{(\eta)} = \hat{G}_n - (n+1)\eta,
\label{eq:calibration_penalty}
\end{equation}
where $\hat{G}_n$ is the uncorrected score. A calibration error of $\eta$ costs $(n+1)\eta$ in Fr\'{e}chet margin.
\end{corollary}

Eq.~\ref{eq:calibration_penalty} explains why aggressive Fr\'{e}chet decoding can fail under overconfidence: the penalty scales linearly with the commit set size. This motivates either calibrating the model or using the robust selector with an explicit $\eta$ budget.

\section{Experimental Evaluation Results on Dream}
\label{app:dream-results}

Table~\ref{tab:dream-results} evaluates Fr\'{e}chet profile decoding on Dream-v0-Base-7B under PrefixCache with block size 32. All runs use a single H100 GPU. Improvements are reported relative to threshold decoding.

\begin{table*}[t]
\centering
\small
\caption{
{\bf Benchmark results on Dream-v0-Base-7B.} (PrefixCache, block size 32, H100).
Throughput improvements are reported as speedup over threshold decoding.
NFE reduction is relative to threshold. NA indicates runs where NFE
logging was not available. 
}
\label{tab:dream-results}
\begin{tabular}{llccccc}
\toprule
Dataset & Len. & Method & Acc.\ (\%) & Tok/s $\uparrow$ & NFE $\downarrow$ & Tok/NFE \\
\midrule
\multirow{3}{*}{GSM8K (5-shot)}
& 256 & Threshold & 73.7 & 51.1 (\speedup{1.00}) & 202{,}862 & 1.66 \\
& 256 & Factor & 72.0 & 63.1 (\speedup{1.23}) & 176{,}073 (\nfered{13.2}) & 1.89 \\
& 256 & Fr\'{e}chet & 72.0 & \textbf{63.8} (\speedup{1.25}) & \textbf{171{,}925} (\nfered{15.3}) & \textbf{1.96} \\
\midrule
\multirow{3}{*}{MATH (4-shot)}
& 256 & Threshold & 34.0 & 72.6 (\speedup{1.00}) & 568{,}643 & 2.25 \\
& 256 & Factor & 33.2 & 92.4 (\speedup{1.27}) & 442{,}791 (\nfered{22.1}) & 2.80 \\
& 256 & Fr\'{e}chet & 33.1 & \textbf{97.0} (\speedup{1.34}) & \textbf{433{,}971} (\nfered{23.7}) & \textbf{2.95} \\
\midrule
\multirow{3}{*}{HumanEval (0-shot)}
& 256 & Threshold & 37.8 & 53.7 (\speedup{1.00}) & 25{,}465 & 1.64 \\
& 256 & Factor & 36.0 & 68.3 (\speedup{1.27}) & 20{,}246 (\nfered{20.5}) & 2.00 \\
& 256 & Fr\'{e}chet & 35.4 & \textbf{70.6} (\speedup{1.31}) & \textbf{20{,}156} (\nfered{20.8}) & \textbf{2.14} \\
\midrule
\multirow{3}{*}{MBPP (3-shot)}
& 256 & Threshold & 56.0 & 85.1 (\speedup{1.00}) & 47{,}070 & 2.72 \\
& 256 & Factor & 51.8 & 106.6 (\speedup{1.25}) & 38{,}043 (\nfered{19.2}) & 3.43 \\
& 256 & Fr\'{e}chet & 53.2 & \textbf{111.2} (\speedup{1.31}) & \textbf{35{,}170} (\nfered{25.3}) & \textbf{3.64} \\
\midrule
\multirow{3}{*}{GSM8K (5-shot)}
& 512 & Threshold & 73.7 & 57.2 & 355{,}339 & 1.90 \\
& 512 & Factor & 73.1 & 70.5 (\speedup{1.23}) & 302{,}274 (\nfered{14.9}) & 2.23 \\
& 512 & Fr\'{e}chet & 72.7 & \textbf{70.7} (\speedup{1.24}) & \textbf{293{,}029} (\nfered{17.5}) & \textbf{2.30} \\
\midrule
\multirow{3}{*}{MATH (4-shot)}
& 512 & Threshold & 35.3 & 99.1 & 776{,}016 & 3.30 \\
& 512 & Factor & 34.1 & 118.6 (\speedup{1.20}) & 623{,}643 (\nfered{19.6}) & 4.10 \\
& 512 & Fr\'{e}chet & 33.5 & \textbf{127.5} (\speedup{1.29}) & \textbf{595{,}000} (\nfered{23.3}) & \textbf{4.30} \\
\midrule
\multirow{3}{*}{HumanEval (0-shot)}
& 512 & Threshold & 37.8 & 58.7 (\speedup{1.00}) & 44{,}862 & 1.87 \\
& 512 & Factor & 36.6 & 71.0 (\speedup{1.21}) & 37{,}740 (\nfered{15.9}) & 2.22 \\
& 512 & Fr\'{e}chet & 37.2 & \textbf{72.7} (\speedup{1.24}) & \textbf{37{,}168} (\nfered{17.2}) & \textbf{2.25} \\
\midrule
\multirow{3}{*}{MBPP (3-shot)}
& 512 & Threshold & 53.2 & 117.6 (\speedup{1.00}) & 63{,}108 & 4.06 \\
& 512 & Factor & 52.8 & 141.5 (\speedup{1.20}) & 51{,}779 (\nfered{18.0}) & 4.94 \\
& 512 & Fr\'{e}chet & 52.8 & \textbf{153.2} (\speedup{1.30}) & \textbf{47{,}709} (\nfered{24.4}) & \textbf{5.37} \\
\bottomrule
\end{tabular}
\end{table*}

\paragraph{Dream efficiency gains.}
On Dream-v0-Base-7B, Fr\'{e}chet decoding achieves \textbf{1.28$\times$} average throughput over threshold decoding and \textbf{21.1\%} NFE reduction (averaged over the six settings where NFE is available), with a 1.45-point average accuracy change. Relative to the Dream full-step baseline, Fr\'{e}chet achieves \textbf{2.80$\times$} average throughput and \textbf{61.6\%} NFE reduction. The largest gains occur on MATH and MBPP where the confidence profiles are most heterogeneous; on GSM8K the speedup is more modest (${\sim}1.25\times$) due to Dream's concentrated confidence distribution on this task.

\section{Additional Experiments}
\label{app:additional-experiments}

\subsection{Effect of prefill length and cache mode}

We next evaluate whether Fr\'{e}chet profile decoding remains effective in the
long-generation regime studied by Fast-dLLM. Following the Table~4 setup of
Fast-dLLM, we compare no cache, PrefixCache, and DualCache under both 5-shot
and 8-shot prompting on GSM8K with generation length 1024. Unlike the original
table, which compares cache variants under a single selector, our goal is to
isolate the effect of the token-selection rule while holding the cache mode
fixed.

Table~\ref{tab:cache_selector_5shot_8shot} shows that Fr\'{e}chet profile
decoding consistently improves decoding efficiency across cache modes. For
5-shot prompting, Fr\'{e}chet improves throughput over factor by $5.7\%$ on
average and improves Tok/NFE by $6.4\%$, while increasing average accuracy by
$1.18$ points. For 8-shot prompting, Fr\'{e}chet improves throughput over factor
by $4.1\%$ and Tok/NFE by $5.6\%$, again with a positive average accuracy
change of $0.89$ points. Compared to threshold decoding, the throughput gains
are larger: $28.2\%$ on average for 5-shot and $24.4\%$ for 8-shot.

These gains come from the selector rather than the cache mechanism: within each
cache mode, all methods use the same attention reuse strategy. The consistent
increase in Tok/NFE indicates that Fr\'{e}chet commits more tokens per denoising
model call, reducing the number of required function evaluations. This supports
our theoretical interpretation that profile-aware decoding extracts additional
safe parallelism beyond weakest-token factor decoding.

\begin{table*}[t]
\centering
\small
\setlength{\tabcolsep}{4.5pt}
\caption{
Selector comparison under different cache modes for LLaDA-8B on GSM8K with
generation length 1024. We report accuracy, throughput, and tokens per function
evaluation (Tok/NFE) for 5-shot and 8-shot settings. Within each cache mode and
shot setting, Fr\'{e}chet achieves the highest throughput and Tok/NFE while
maintaining comparable accuracy.
}
\label{tab:cache_selector_5shot_8shot}
\begin{tabular}{llcccccc}
\toprule
\multirow{2}{*}{Cache} & \multirow{2}{*}{Selector} &
\multicolumn{3}{c}{5-shot} & \multicolumn{3}{c}{8-shot} \\
\cmidrule(lr){3-5} \cmidrule(lr){6-8}
& & Acc. & Tok/s & Tok/NFE & Acc. & Tok/s & Tok/NFE \\
\midrule
\multirow{3}{*}{No cache}
& Threshold $\tau{=}0.9$     & \textbf{78.33} & 25.56 & 2.33 & 79.84 & 19.99 & 2.32 \\
& Factor $f{=}1.0$           & 76.67 & 32.37 & 2.92 & 79.65 & 25.01 & 2.87 \\
& Fr\'{e}chet $\delta{=}0.25$ & \textbf{78.33} & \textbf{35.08} & \textbf{3.16} & \textbf{80.33} & \textbf{27.11} & \textbf{3.09} \\
\midrule
\multirow{3}{*}{PrefixCache}
& Threshold $\tau{=}0.9$     & \textbf{78.33} & 38.48 & 2.31 & \textbf{77.33} & 32.38 & 2.25 \\
& Factor $f{=}1.0$           & \textbf{78.33} & 44.76 & 2.87 & 76.67 & 38.17 & 2.87 \\
& Fr\'{e}chet $\delta{=}0.25$ & 78.21 & \textbf{46.86} & \textbf{3.03} & \textbf{77.33} & \textbf{38.73} & \textbf{3.00} \\
\midrule
\multirow{3}{*}{DualCache}
& Threshold $\tau{=}0.9$     & 76.67 & 37.66 & 2.17 & 77.33 & 34.56 & 2.21 \\
& Factor $f{=}1.0$           & 76.33 & 45.52 & 2.76 & 76.67 & 39.84 & 2.76 \\
& Fr\'{e}chet $\delta{=}0.25$ & \textbf{78.33} & \textbf{47.34} & \textbf{2.91} & \textbf{78.00} & \textbf{40.77} & \textbf{2.89} \\
\bottomrule
\end{tabular}
\end{table*}

\begin{figure*}[t]
\centering
\includegraphics[width=0.85\linewidth]{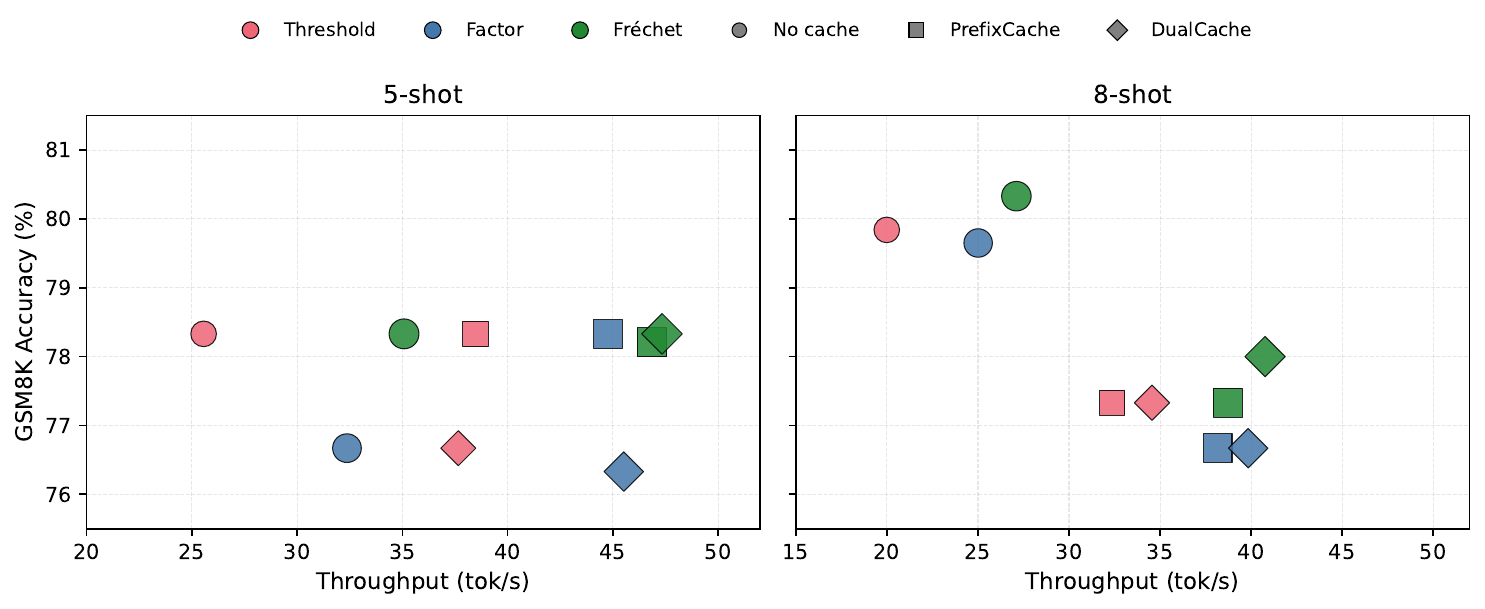}
\caption{
Accuracy--throughput trade-off across cache modes for 5-shot and 8-shot GSM8K
at generation length 1024. Marker shape encodes cache mode; color encodes
selector method; marker size is proportional to Tok/NFE. Fr\'{e}chet profile
decoding (green) consistently achieves the highest throughput and Tok/NFE in
every cache mode while maintaining comparable or better accuracy.
}
\label{fig:table4-selector-scatter}
\end{figure*}

\subsection{Impact of generation length}

Table~\ref{tab:gen-length-scaling} shows how throughput scales with generation
length (256, 512, 1024) under 8-shot GSM8K for LLaDA-8B. The original
Fast-dLLM paper reports only threshold decoding in this setting; we extend the
comparison to factor and Fr\'{e}chet. Fr\'{e}chet consistently achieves the
highest throughput at every generation length and cache mode, with 5--20\%
speedup over factor and 20--38\% over threshold. 
The NFE reduction is 6--8\%
vs.\ factor and 25--33\% vs.\ threshold across all settings. The heterogeneity
bonus depends on the block size (Figure~\ref{fig:block-ablation},
Appendix~\ref{app:v2}), not on the generation length; the throughput gain over
threshold grows with length because the cache cost is amortized over more steps.

\begin{table*}[t]
\centering
\small
\setlength{\tabcolsep}{4pt}
\caption{
Impact of generation length on accuracy and throughput under 8-shot GSM8K
(LLaDA-8B, block size 32). Fr\'{e}chet achieves the highest throughput and
lowest NFE at every generation length and cache mode. 512-length throughput is
All runs use a single H100 GPU.
}
\label{tab:gen-length-scaling}
\begin{tabular}{clcccccc}
\toprule
& & \multicolumn{3}{c}{\textbf{PrefixCache}} & \multicolumn{3}{c}{\textbf{DualCache}} \\
\cmidrule(lr){3-5} \cmidrule(lr){6-8}
\textbf{Len.} & \textbf{Metric}
& Threshold & Factor & Fr\'{e}chet
& Threshold & Factor & Fr\'{e}chet \\
\midrule
\multirow{3}{*}{256}
& Acc.\ (\%) & 77.0 & 75.7 & 76.4 & 77.3 & 75.7 & 75.4 \\
& Tok/s      & 69.8 & 90.8 & \textbf{96.1} & 56.4 & 75.4 & \textbf{80.9} \\
& NFE        & 109,644 & 80,641 & \textbf{74,289} & 115,196 & 85,283 & \textbf{78,901} \\
\midrule
\multirow{3}{*}{512}
& Acc.\ (\%) & 76.1 & 76.4 & 77.6 & 75.2 & 73.8 & 75.8 \\
& Tok/s      & 37.8 & 43.9 & \textbf{49.2} & 40.1 & 45.9 & \textbf{50.4} \\
& NFE        & 132,492 & 101,083 & \textbf{93,936} & 139,106 & 105,227 & \textbf{102,145} \\
\midrule
\multirow{3}{*}{1024}
& Acc.\ (\%) & 77.3 & 76.7 & 77.3 & 77.3 & 76.7 & 78.0 \\
& Tok/s      & 32.4 & 38.2 & \textbf{38.7} & 34.6 & 39.8 & \textbf{40.8} \\
& NFE        & 31,918 & 25,547 & \textbf{24,047} & 33,308 & 26,609 & \textbf{25,172} \\
\bottomrule
\end{tabular}
\end{table*}
\section{Results on Multimodal Diffusion Language models}
\label{sec:exp_dmllm}

\begin{table*}[h!]
\centering
\small
\caption{\textbf{Performance and Speedup Comparison of LLaDA-V on MathVista and MathVerse.}
We compare full-step decoding, half-step decoding, Fast-dLLM, and our method. 
Throughput is reported with speedup relative to full-step decoding on the same benchmark.}
\setlength{\tabcolsep}{6pt}
\resizebox{\textwidth}{!}{
\begin{tabular}{lllll}
\toprule
\multirow{2}{*}{\textbf{Method}} 
& \multicolumn{2}{c}{\textbf{MathVista}} 
& \multicolumn{2}{c}{\textbf{MathVerse}} \\
\cmidrule(lr){2-3} \cmidrule(lr){4-5}
& Accuracy (\%) & Throughput (Speedup)
& Accuracy (\%) & Throughput (Speedup) \\
\midrule

Full Steps  & 59.2 
            & {2.84} (\speedup{1})
            & 28.5 
            & {2.75} (\speedup{1}) \\

Half Steps  & {59.7} 
            & {5.56} (\speedup{1.96}) 
            & 28.3 
            & {5.17} (\speedup{1.88}) \\

Fast-dLLM   & 56.6 
            & {28.2} (\speedup{9.9}) 
            & {28.6}
            & {23.3} (\speedup{8.5}) \\

{\bf Ours}  & 56.8
            & {\bf 32.9} (\speedup{11.6}) 
            & 27.9
            & {\bf 23.8} (\speedup{8.66}) \\

\bottomrule
\end{tabular}
}
\label{tab:lladav_main_results}
\end{table*}

We further evaluate Fast-dLLM++ on LLaDA-V to test whether profile-aware decoding transfers to multimodal diffusion language models. On MathVista, Fast-dLLM++ improves throughput from 28.2 to 32.9 tokens/s over Fast-dLLM, increasing speedup from 9.9× to 11.6× relative to full-step decoding, while slightly improving accuracy from 56.6\% to 56.8\%. This suggests that Fréchet profile decoding can exploit heterogeneous confidence profiles even in visual mathematical reasoning. On MathVerse, Fast-dLLM++ provides a smaller efficiency gain, improving throughput from 23.3 to 23.8 tokens/s, while accuracy decreases from 28.6\% to 27.9\%. These results indicate that the proposed selector generalizes to multimodal diffusion LLMs, but the magnitude of the gain is task-dependent. In particular, MathVista shows a clear accuracy-preserving efficiency improvement over Fast-dLLM, whereas MathVerse suggests that multimodal reasoning with stronger token or visual-text dependencies may require more conservative margins or calibration-aware selection.

\definecolor{baselinecolor}{RGB}{217, 95, 2} 
\definecolor{ourscolor}{RGB}{27, 119, 184}   
\definecolor{bestcolor}{RGB}{35, 139, 69}    

\begin{figure}[htbp]
\centering
\resizebox{\textwidth}{!}{
\newcommand{\figscale}{1.50}
\begin{tikzpicture}[scale=\figscale, every node/.style={scale=\figscale}]

\begin{axis}[
    name=plot1,
    width=12cm, height=9.5cm,
    xlabel={Fr\'echet Margin $\delta$},
    xmin=-0.02, xmax=0.32,
    xtick={0.00, 0.05, 0.10, 0.15, 0.20, 0.30},
    axis y line*=left,
    ylabel={Speedup vs.\ full-step (64 NFE)},
    ymin=1.48, ymax=1.80,
    ytick={1.50, 1.55, 1.60, 1.65, 1.70, 1.75, 1.80},
    yticklabel={$\pgfmathprintnumber{\tick}\times$},
    ylabel style={ourscolor},
    yticklabel style={ourscolor},
    grid=major,
    grid style={dashed, gray!30},
    title={\textbf{(a) Speedup and Accuracy vs.\ Fr\'echet margin}},
    title style={font=\large},
    label style={font=\large},
    tick label style={font=\small},
    legend style={
        legend pos=north east,
        legend columns=1,
        font=\footnotesize,
        draw=gray!0,
        /tikz/every even column/.append style={column sep=0.4cm},
    },
    legend cell align={left},
]

\addplot[domain=-0.02:0.32, samples=2, ourscolor, thick, dashed, forget plot] {1.52};
\node[anchor=west, font=\small, text=ourscolor, rotate=0]
    at (axis cs:-0.02, 1.505) {baseline: $1.52\times$};

\addplot[ourscolor, thick, mark=*, mark size=3pt,line width=2pt] coordinates {
    (0.00, 1.76)
    (0.05, 1.72)
    (0.10, 1.68)
    (0.15, 1.65)
    (0.20, 1.61)
    (0.30, 1.54)
};
\addlegendentry{Speedup - Fast-dLLM++}

\end{axis}

\begin{axis}[
    at=(plot1.south west), anchor=south west,
    width=12cm, height=9.5cm,
    xmin=-0.02, xmax=0.32,
    xtick=\empty,
    axis y line*=right,
    axis x line=none,
    ylabel={Accuracy $\uparrow$},
    ymin=0.318, ymax=0.335,
    ytick={0.320, 0.325, 0.330, 0.335},
    yticklabel style={/pgf/number format/.cd, fixed, precision=3, /tikz/.cd, bestcolor},
    ylabel style={bestcolor},
    label style={font=\large},
    tick label style={font=\small},
    legend style={
        at={(0.72,0.90)}, anchor=north,
        legend columns=1,
        font=\footnotesize,
        draw=gray!0,
        /tikz/every even column/.append style={column sep=0.4cm},
    },
    legend cell align={left},
]

\addplot[domain=-0.02:0.32, samples=2, bestcolor, thick, dashed, forget plot] {0.3299};
\node[anchor=west, font=\small, text=bestcolor, rotate=0]
    at (axis cs:-0.02, 0.329) {baseline: $0.329$};
\addplot[bestcolor, thick, mark=square*, mark size=3pt,line width=2pt] coordinates {
    (0.00, 0.3223)
    (0.05, 0.3211)
    (0.10, 0.3261)
    (0.15, 0.3287)
    (0.20, 0.3312)
    (0.30, 0.3287)
};
\addlegendentry{Accuracy - Fast-dLLM++}

\end{axis}

\begin{axis}[
    name=plot2,
    at={(plot1.right of east)}, anchor=left of west,
    xshift=2.0cm,
    width=12cm, height=9.5cm,
    xlabel={Fr\'echet Margin $\delta$},
    xmin=-0.02, xmax=0.32,
    xtick={0.00, 0.05, 0.10, 0.15, 0.20, 0.30},
    axis y line*=left,
    ylabel={Avg.\ NFE $\downarrow$},
    ymin=35.5, ymax=44,
    ytick={36, 38, 40, 42},
    ylabel style={ourscolor},
    yticklabel style={ourscolor},
    grid=major,
    grid style={dashed, gray!30},
    title={\textbf{(b) Avg.\ NFE and Tokens/NFE vs.\ Fr\'echet margin}},
    title style={font=\large},
    label style={font=\large},
    tick label style={font=\small},
    legend style={
        at={(0.68,0.98)}, anchor=north,
        legend columns=1,
        font=\footnotesize,
        draw=gray!0,
        /tikz/every even column/.append style={column sep=0.4cm},
    },
    legend cell align={left},
]

\addplot[domain=-0.02:0.32, samples=2, ourscolor, thick, dashed, forget plot] {42.16};
\node[anchor=west, font=\small, text=ourscolor, rotate=0]
    at (axis cs:-0.02, 42.6) {baseline: $42.16$};

\addplot[ourscolor, thick, mark=*, mark size=3pt, line width=2pt] coordinates {
    (0.00, 36.38)
    (0.05, 37.18)
    (0.10, 38.04)
    (0.15, 38.89)
    (0.20, 39.81)
    (0.30, 41.53)
};
\addlegendentry{Avg. NFE - Fast-dLLM++}

\end{axis}

\begin{axis}[
    at=(plot2.south west), anchor=south west,
    width=12cm, height=9.5cm,
    xmin=-0.02, xmax=0.32,
    xtick=\empty,
    axis y line*=right,
    axis x line=none,
    ylabel={Tokens / NFE $\uparrow$},
    ymin=1.48, ymax=1.85,
    ytick={1.50, 1.55, 1.60, 1.65, 1.70, 1.75, 1.80},
    ylabel style={bestcolor},
    yticklabel style={bestcolor},
    label style={font=\large},
    tick label style={font=\small},
    legend style={
        at={(0.68,0.90)}, anchor=north,
        legend columns=2,
        font=\footnotesize,
        draw=gray!0,
        /tikz/every even column/.append style={column sep=0.4cm},
    },
    legend cell align={left},
]

\addplot[domain=-0.02:0.32, samples=2, bestcolor, thick, dotted, forget plot] {1.518};
\node[anchor=west, font=\small, text=bestcolor, rotate=0]
    at (axis cs:-0.02, 1.50) {baseline: $1.518$};

\addplot[bestcolor, thick, mark=square*, mark size=3pt, line width=2pt] coordinates {
    (0.00, 1.759)
    (0.05, 1.721)
    (0.10, 1.682)
    (0.15, 1.645)
    (0.20, 1.608)
    (0.30, 1.541)
};
\addlegendentry{Tokens/NFE - Fast-dLLM++}

\end{axis}

\end{tikzpicture}
}
    \caption{{\bf Ablation study of Fr\'echet margin.} We evaluate Fast-dLLM++ on MathVista, a multimodal math reasoning using LLaDA-V~\cite{you2025lladav}.}
    \label{fig:lladav_ablation}
\end{figure}

Figure~\ref{fig:lladav_ablation} studies the effect of the Fréchet margin $\delta$ on MathVista with LLaDA-V. The margin controls the aggressiveness of profile-aware commitment: smaller $\delta$ values allow larger parallel commit sets, while larger values enforce a stricter certificate. The results show a clear monotonic efficiency trend. As $\delta$ increases, speedup decreases and average NFE increases, confirming that larger margins make decoding more conservative. Nevertheless, all tested margins retain higher speedup and Tokens/NFE than the Fast-dLLM baseline, indicating that the Fréchet selector provides a robust efficiency advantage. Accuracy exhibits a moderate-margin optimum: aggressive settings such as $\delta=0$ and $\delta=0.05$ improve speed but can overcommit, whereas $\delta\approx0.15 - 
0.20$ offers the best balance, matching or slightly exceeding the baseline accuracy while preserving substantial speedup. This supports using a moderate margin as the default setting for multimodal mathematical reasoning.

\begin{figure}[htbp]
\centering
\begin{tcolorbox}[
  colback=gray!10, colframe=gray!60,
  boxrule=0.4pt, arc=2pt,
  left=8pt, right=8pt, top=8pt, bottom=8pt
]
\begin{minipage}[t]{0.30\textwidth}
  \vspace{0pt}
  \includegraphics[width=\linewidth]{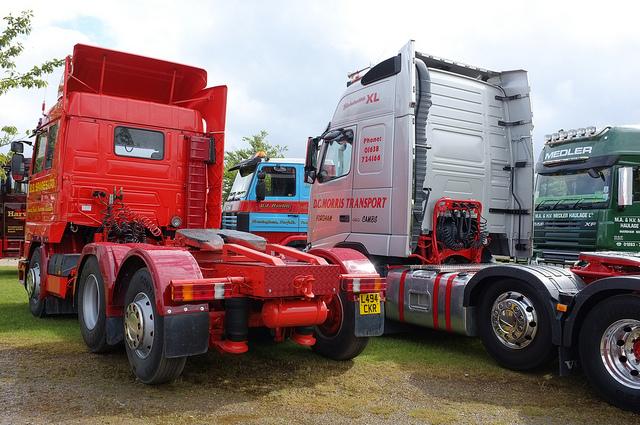}
\end{minipage}%
\hfill
\begin{minipage}[t]{0.67\textwidth}
  \vspace{0pt}
  \sloppy
  \footnotesize
  \textbf{User:} Please carefully observe the image and come up with a brief caption.

  \vspace{4pt}
  \textbf{Fast-dLLM} (\textcolor{red!70!black}{\textbf{12.38s}}):
  Several large truck cabs are parked closely together on a grassy lot at an outdoor display. A bright red cab is on the left, a silver transport cab is on the right, and more trucks are visible in the background.
  \vspace{4pt}
  
  \textbf{Fast-dLLM++ (Ours)} (\textcolor{green!50!black}{\textbf{9.86s}}):
  Several large truck cabs are lined up on grass at an outdoor vehicle display. The scene features a bright red cab on the left, a silver transport cab on the right, and additional trucks parked behind them.

\end{minipage}
\end{tcolorbox}
\caption{\textbf{Visual grounding example \#1.} Fast-dLLM++ produces comparable caption quality while reducing decoding time from 12.38s to 9.86s.}
\label{fig:visual_grounding_2}
\end{figure}

\begin{figure}[htbp]
\centering
\begin{tcolorbox}[
  colback=gray!10, colframe=gray!60,
  boxrule=0.4pt, arc=2pt,
  left=8pt, right=8pt, top=8pt, bottom=8pt
]
\begin{minipage}[t]{0.30\textwidth}
  \vspace{0pt}
  \includegraphics[width=\linewidth]{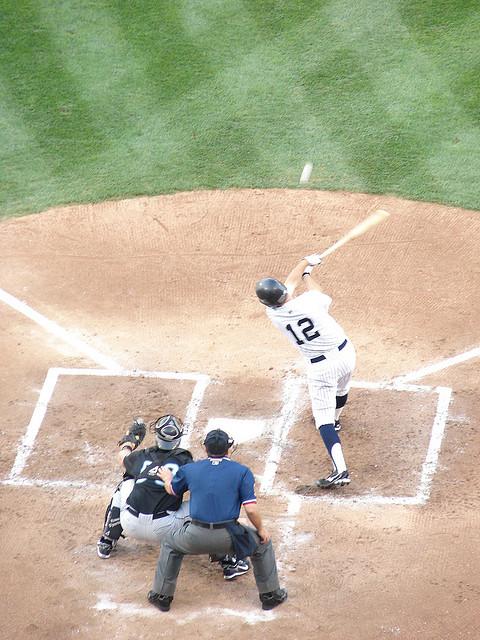}
\end{minipage}%
\hfill
\begin{minipage}[t]{0.67\textwidth}
  \vspace{0pt}
  \sloppy
  \footnotesize
  \textbf{User:} Please carefully observe the image and come up with a detailed caption.

  \vspace{4pt}
  {\bf Fast-dLLM (\textcolor{red!70!black}{18.12s})}: The image shows a baseball game from a high, overhead viewpoint focused on home plate. A right-handed batter in a white uniform with the number 12 is following through after swinging the bat. The ball can be seen rising toward the outfield grass, suggesting he has just hit it. The catcher, dressed in black protective equipment, is crouched behind the plate, and the umpire stands directly behind him. Bright chalk lines mark the batter's boxes and foul lines, contrasting with the dirt and green grass.
  
  \vspace{4pt}
  {\bf Fast-dLLM++ (Ours) (\textcolor{green!50!black}{14.75s})}: A baseball player in a white pinstriped uniform, wearing number 12, has just completed a powerful swing at home plate. The ball is visible traveling upward over the dirt area near the edge of the grass. Behind the batter, the catcher in dark gear remains crouched, while the umpire stands closely behind him, watching the play. The image is taken from an elevated angle, clearly showing the batter's box, home plate, chalk lines, and surrounding infield dirt.
\end{minipage}
\end{tcolorbox}
\caption{\textbf{Visual grounding example \#2.} Fast-dLLM++ produces comparable caption quality while reducing decoding time from 18.12s to 14.75s.}
\label{fig:visual_grounding_1}
\end{figure}

\definecolor{CorrectGreen}{HTML}{0B6B3A}
\definecolor{WrongRed}{HTML}{A51C30}
\definecolor{NeutralGray}{HTML}{555555}
\definecolor{LightGray}{HTML}{F7F7F7}

\newcommand{\Frechet}{Fr\'echet}
\newcommand{\CorrectTag}{\textcolor{CorrectGreen}{\textsc{Correct}}}
\newcommand{\WrongTag}{\textcolor{WrongRed}{\textsc{Wrong}}}
\newcommand{\MixedTag}{\textcolor{NeutralGray}{\textsc{Mixed}}}
\newcommand{\Placeholder}[1]{\textcolor{NeutralGray}{[#1]}}

\lstdefinestyle{modeltext}{
  basicstyle=\ttfamily\scriptsize,
  breaklines=true,
  breakatwhitespace=false,
  columns=fullflexible,
  keepspaces=true,
  showstringspaces=false,
  tabsize=2
}

\lstdefinestyle{modelcode}{
  style=modeltext,
  language=Python
}

\tcbset{
  modelpanelbase/.style={
    enhanced,
    listing only,
    breakable,
    colback=white,
    colframe=black!20,
    coltitle=black,
    fonttitle=\bfseries\footnotesize,
    boxrule=0.4pt,
    arc=1pt,
    left=1mm,
    right=1mm,
    top=1mm,
    bottom=1mm,
    before skip=0.25em,
    after skip=0.25em
  },
  analysisbase/.style={
    enhanced,
    breakable,
    colback=LightGray,
    colframe=black!15,
    boxrule=0.4pt,
    arc=1pt,
    left=1.2mm,
    right=1.2mm,
    top=1mm,
    bottom=1mm,
    before skip=0.5em,
    after skip=0.5em
  }
}

\newtcblisting{modeltextbox}[3][]{%
  modelpanelbase,
  title={#2 \hfill {\normalfont #3}},
  listing options={style=modeltext},
  #1
}

\newtcblisting{modelcodebox}[3][]{%
  modelpanelbase,
  title={#2 \hfill {\normalfont #3}},
  listing options={style=modelcode},
  #1
}

\newtcolorbox{analysisbox}[1][]{%
  analysisbase,
  #1
}

\section{Qualitative Analysis of Text Generation Results}
\label{app:text-generation-qualitative}

We provide qualitative examples of decoded text generated by
\texttt{Threshold}, \texttt{Factor}, and \texttt{\Frechet{}} decoding. The goal is to
isolate where the methods diverge textually, when those divergences affect
answer correctness, and whether the reasoning template is preserved across
decoding methods.


\begin{figure}[t]
    \centering
    \includegraphics[width=\linewidth]{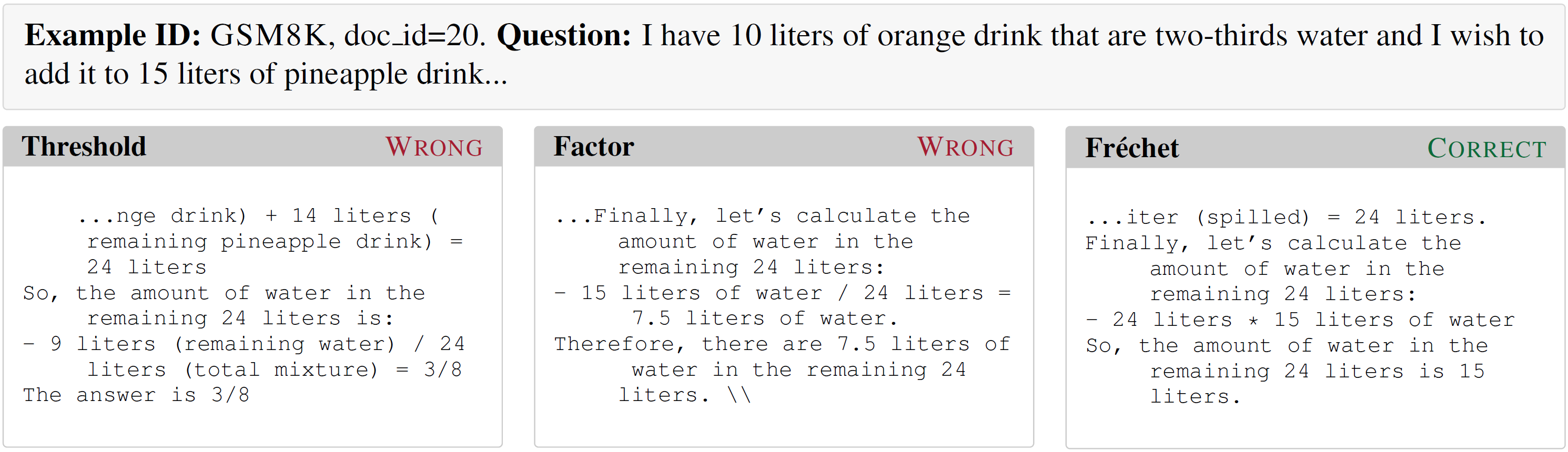}
    \caption{
    {\bf Representative GSM8K example \#1.} \Frechet{} decoding produces the correct
    final answer while the alternative methods make arithmetic-level errors.
  }
  \label{fig:app-gsm8k-frechet-correct}
\end{figure}


\begin{figure}[!h]
    \centering
    \includegraphics[width=\linewidth]{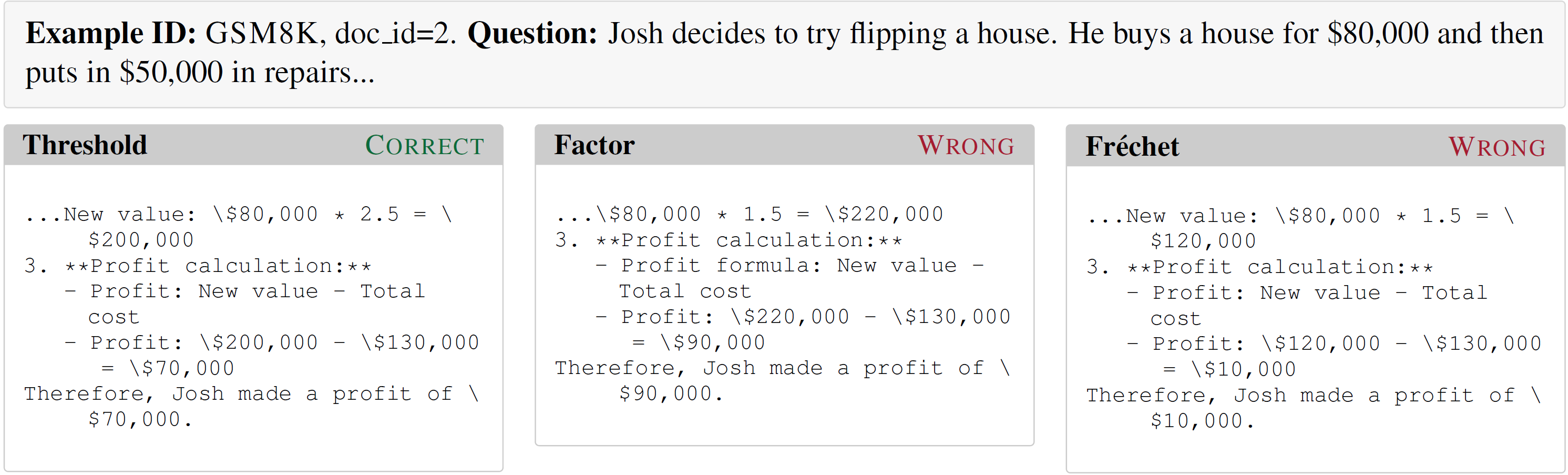}
    \caption{
    {\bf Representative GSM8K example \#2.} Threshold decoding preserves the correct
    numerical path while the other decoding methods commit to incorrect
    intermediate values.
    }
    \label{fig:app-gsm8k-threshold-correct}
\end{figure}

\paragraph{Summary.}
Across GSM8K, we observe that decoded outputs are textually different in most
samples, but only a much smaller fraction differ in final correctness. In
particular, 1204 out of 1319 GSM8K samples, or 91.3\%, have textually different
outputs across decoding methods, while only 171 out of 1319 samples, or 13.0\%,
differ in correctness. This indicates that the methods usually preserve the
same high-level reasoning structure, while diverging on local numerical values,
intermediate arithmetic, or final formatting.

The dominant divergence pattern is that the methods agree on the reasoning
template, including sentence structure and step-by-step decomposition, but
commit to different specific digits or intermediate values. This is consistent
with the failure mode of parallel token decoding: tokens that should be
conditionally dependent, such as digits within a number or related arithmetic
quantities, may be decoded simultaneously.

\paragraph{Interpretation.}
The large gap between textual divergence and correctness divergence suggests
that most method-level differences are not semantic failures. Instead, the
methods often preserve the same problem-solving template while varying in
surface form, formatting, or local token choices. Correctness changes are
concentrated in low-confidence regions, especially numerical tokens in GSM8K
and implementation details in MBPP.


\begin{figure}[t]
  \centering
  \includegraphics[width=\linewidth]{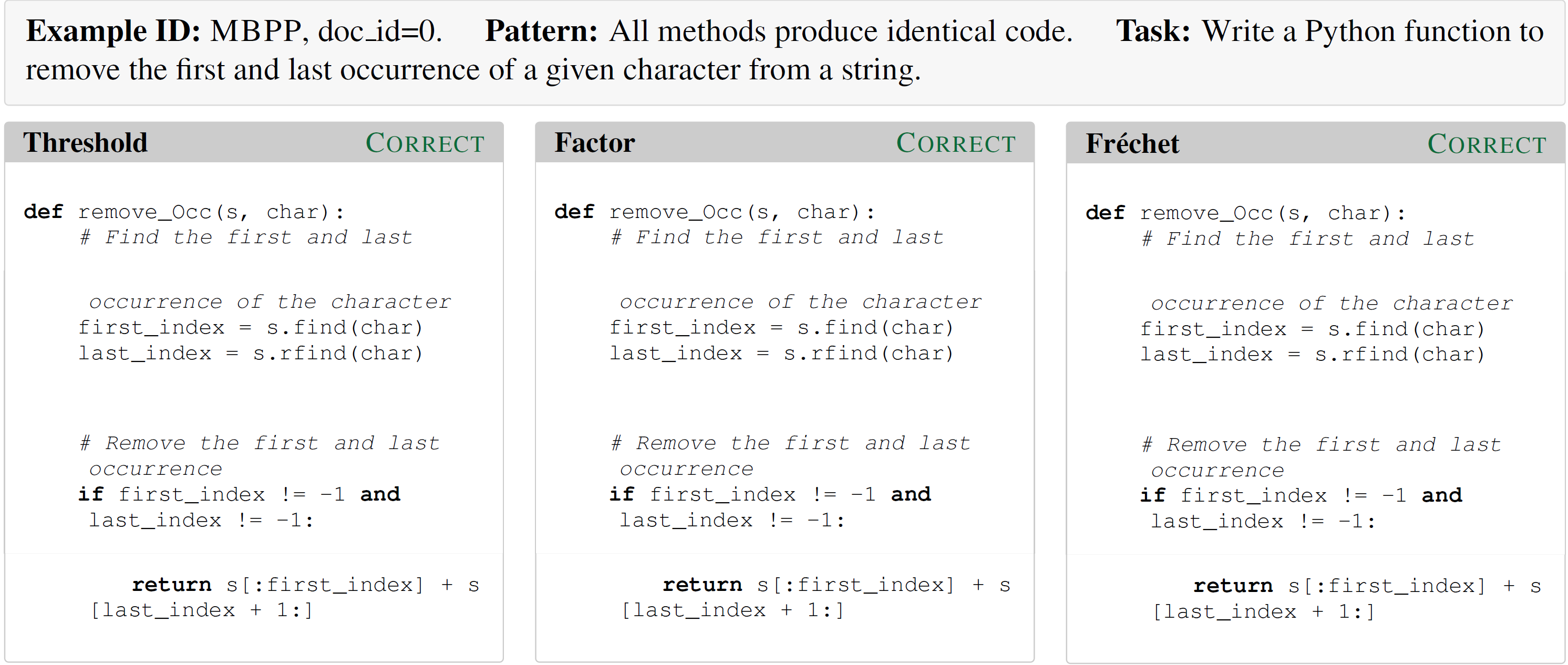}
  \caption{
    MBPP example where all decoding methods produce identical code.
  }
  \label{fig:app-mbpp-identical-code-doc0}
\end{figure}


\section{Matched Fr\'{e}chet versus Factor Sweeps}\label{app:matched-sweeps}

Since Fr\'{e}chet with margin $\delta$ is theoretically matched to factor with
$f=1-\delta$ (Corollary~\ref{cor:factor-reduction}), comparing the two directly
isolates the gain attributable to the heterogeneous confidence profile rather
than to a different operating point. Table~\ref{tab:main-results} compares
against factor at its most aggressive operating point $f=1.0$; here we extend
the matched comparison to MATH, HumanEval, PrefixCache, DualCache, and
Fast-dLLM-v2, the configurations not covered by Figure~\ref{fig:pareto-frontier}.
F denotes Fr\'{e}chet($\delta$) and C the matched factor $(f=1-\delta)$; the
speed improvement given by Fr\'{e}chet is measured as the relative improvement
in NFE and tok/s.

\begin{table}[h]
\centering\small
\caption{Matched sweeps on LLaDA-8B (gen 256, block 32) and Fast-dLLM-v2 7B. Accuracy F/C in percent.}
\label{tab:app-matched-sweeps}
\begin{tabular}{lcccc}
\toprule
Setting & $\delta$ & acc F/C & NFE cut & tok/s gain \\
\midrule
\multirow{5}{*}{MATH, PrefixCache}
 & 0.00 & 27.6/32.7 & 18\% & $+22\%$ \\
 & 0.10 & 32.4/32.7 & 17\% & $+20\%$ \\
 & 0.20 & 32.5/31.2 & 15\% & $+17\%$ \\
 & 0.25 & 32.5/32.5 & 14\% & $+17\%$ \\
 & 0.30 & 31.9/32.8 & 14\% & $+16\%$ \\
\midrule
\multirow{5}{*}{HumanEval, PrefixCache}
 & 0.00 & 39.6/40.7 & 20\% & $+25\%$ \\
 & 0.10 & 41.5/40.7 & 17\% & $+21\%$ \\
 & 0.20 & 41.5/41.5 & 14\% & $+16\%$ \\
 & 0.25 & 40.9/40.1 & 13\% & $+14\%$ \\
 & 0.30 & 43.9/40.8 & 16\% & $+19\%$ \\
\midrule
\multirow{5}{*}{HumanEval, DualCache}
 & 0.00 & 32.9/33.5 & 19\% & $+23\%$ \\
 & 0.10 & 35.4/34.2 & 16\% & $+19\%$ \\
 & 0.20 & 33.5/34.8 & 14\% & $+16\%$ \\
 & 0.25 & 36.6/34.8 & 14\% & $+16\%$ \\
 & 0.30 & 36.6/34.8 & 14\% & $+16\%$ \\
\midrule
\multirow{5}{*}{Fast-dLLM-v2, GSM8K, sub-block 4}
 & 0.00 & 37.4/39.4 & 13\% & $+16\%$ \\
 & 0.10 & 38.2/38.0 & 12\% & $+14\%$ \\
 & 0.20 & 39.6/38.8 & 11\% & $+12\%$ \\
 & 0.25 & 38.0/39.2 & 10\% & $+11\%$ \\
 & 0.30 & 38.6/39.0 & 9\% & $+10\%$ \\
\bottomrule
\end{tabular}
\end{table}

Across every additional configuration, matched Fr\'{e}chet holds accuracy within
sampling noise of factor$(f=1-\delta)$ while cutting 9 to 20\% of forward
evaluations, i.e.\ a consistent 10 to 25\% throughput gain. The effect persists
across benchmarks (MATH, HumanEval), cache regimes (PrefixCache, DualCache), and
backbones (LLaDA-8B, Fast-dLLM-v2 7B), confirming that the profile-aware bonus
is not specific to the GSM8K setting in Figure~\ref{fig:pareto-frontier}. As
expected, the margin $\delta$ trades efficiency for conservatism monotonically,
so a smaller $\delta$ can be chosen where accuracy is robust and a larger
$\delta$ where it is not. On MBPP (LLaDA-8B, PrefixCache, full test set),
factor at $f=1.0$ obtains 21.4 pass@1 with 26{,}870 NFE, matched factor
$f=0.75$ obtains 24.0 with 29{,}305 NFE, and Fr\'{e}chet $\delta=0.25$ obtains
25.3 with 26{,}169 NFE (Table~\ref{tab:app-mbpp}).

\paragraph{How often the bonus changes a decision.}
Any refinement of a deterministic decision rule must change a decision at the
old rule's boundary; what matters is how often that happens and what it
compounds to. Table~\ref{tab:app-commit-more} gives how often Fr\'{e}chet
$\delta=0.25$ commits more than matched factor $f=0.75$ on the LLaDA-8B
(256, PrefixCache) setting, and the extra tokens per such step. On roughly a
third to a half of steps Fr\'{e}chet commits more, and each such step removes
future forward passes; these per-step savings accumulate over the full
generation into the NFE reductions reported in Table~\ref{tab:main-results}.

\begin{table}[h]
\centering\small
\caption{Steps on which Fr\'{e}chet $\delta=0.25$ commits more tokens than matched factor $f=0.75$ (LLaDA-8B, gen 256, PrefixCache).}
\label{tab:app-commit-more}
\begin{tabular}{lcc}
\toprule
Task & Steps committing more than factor & Extra tokens per such step \\
\midrule
GSM8K & 41.7\% & 0.70 \\
MATH & 38.2\% & 0.63 \\
HumanEval & 34.9\% & 0.60 \\
MBPP & 31.7\% & 0.50 \\
\bottomrule
\end{tabular}
\end{table}

\section{Commit-Level Calibration and Robustness}\label{app:calibration}

\paragraph{Reference decode.}
We added measurement code to the decoder in the main setting (LLaDA-8B-Instruct,
PrefixCache) so that every parallel commit is checked against a reference: at
each step, the selected positions are also decoded one token at a time,
most-confident first, from the same decoder state. A parallel commit is defined
correct if it produces exactly the same tokens as this one-at-a-time decoding.
For every multi-token commit we also record the three quantities the theory
computes from the confidences alone: $L_n$, the certified lower bound on the
probability that all $n$ committed tokens are jointly correct; $U_n$, the upper
bound on the probability that any competing alternative is correct; and their
difference $G_n=L_n-U_n$, the certified safety margin of the commit. The rule
commits when $G_n>\delta$.

\begin{table}[h]
\centering\small
\caption{Commit correctness and ranking quality of the margin (LLaDA-8B, PrefixCache). AUC measures how well $G_n$ ranks commits from most to least likely to be correct.}
\label{tab:app-commit-correct}
\begin{tabular}{lcc}
\toprule
Selector & GSM8K: commit correct / AUC($G_n$) & MBPP: commit correct / AUC($G_n$) \\
\midrule
Factor $f=1.0$ & 95.9\% / .866 & 95.8\% / .891 \\
Factor $f=0.75$ & 96.5\% / .909 & 97.1\% / .923 \\
Fr\'{e}chet $\delta=0.25$ & 96.9\% / .671 & 97.7\% / .700 \\
\bottomrule
\end{tabular}
\end{table}

We make three observations. (1) The certificate is sound in practice: the
observed fraction of correct commits never falls below the floor $L_n$ in any
margin decile of any run (zero violations across all selectors and tasks). This
is exactly the conservative guarantee the theory claims. (2) The margin is
predictive: AUC measures how well $G_n$ ranks commits from most to least likely
to be correct (1.0 = perfect ranking, .5 = uninformative). Measured over
factor's commits, which span the whole range of margins from risky to safe,
$G_n$ ranks correctness with AUC .87 to .92. Fr\'{e}chet's own AUC is lower for a
simple statistical reason, not a calibration failure: Fr\'{e}chet only ever
accepts commits whose margin already clears the $\delta$ bar, so its accepted
commits all have similar (high) margins, leaving a ranking statistic almost no
variation to work with. (3) Fr\'{e}chet achieves the highest commit correctness
of the three selectors while using the fewest model evaluations.

\paragraph{Calibration-robust variant.}
To evaluate miscalibration specifically, we run the robust variant of
Appendix~\ref{app:certified-frontier} ($\hat c_i\mapsto\max\{0,\hat c_i-\eta\}$,
i.e.\ discount every reported confidence by $\eta$ before applying the
certificate) in the main setting. Commit error is disagreement with the
one-at-a-time reference.

\begin{table}[h]
\centering\small
\caption{Robust Fr\'{e}chet ($\delta=0.25$) under a uniform confidence discount $\eta$ (LLaDA-8B, PrefixCache).}
\label{tab:app-eta}
\begin{tabular}{llcc}
\toprule
Task & $\eta$ & Accuracy (\%) & Commit error \\
\midrule
GSM8K & .00 & 76.0 & 3.12\% \\
GSM8K & .06 & 78.5 & 0.78\% \\
GSM8K & .10 & 78.5 & 0.32\% \\
\midrule
MBPP & .00 & 19.0 & 2.30\% \\
MBPP & .06 & 22.0 & 0.46\% \\
MBPP & .10 & 21.0 & 0.20\% \\
\bottomrule
\end{tabular}
\end{table}

Commit error drops about 10x with accuracy flat to improved; the cost is smaller
parallel commit sets (mean size 7.6 to 5.5 on MBPP at $\eta=.06$), exactly as
Corollary~\ref{cor:calibration-penalty} predicts ($(n+1)\eta$ margin
shrinkage). Confidence miscalibration affects threshold and factor at least as
strongly (factor $f=1.0$ shows the highest commit error in
Table~\ref{tab:app-commit-correct}); it is a property of the underlying model
shared by all confidence-based selectors, and ours is the variant with an
explicit robustness parameter and a quantified certificate.

\paragraph{MBPP: overconfidence or token coupling.}
On MBPP at generation length 256, factor drops sharply (27.4 to 21.4) while
Fr\'{e}chet partially recovers. We ran the diagnostic in the main setting
(LLaDA-8B, MBPP 3-shot, PrefixCache). The drop is driven by factor at its most
aggressive operating point ($f=1.0$); matched factor ($f=0.75$) and Fr\'{e}chet
recover most of it, and the calibration-robust variant recovers it fully.
Table~\ref{tab:app-mbpp} reports full MBPP test-set pass@1; the factor $f=1.0$
row reproduces the value in Table~\ref{tab:main-results}.

\begin{table}[h]
\centering\small
\caption{Full MBPP test set, LLaDA-8B, PrefixCache, gen 256.}
\label{tab:app-mbpp}
\begin{tabular}{lccc}
\toprule
Method & pass@1 & total NFE & tok/s \\
\midrule
Threshold $\tau=0.9$ & 27.4 & 34{,}975 & 66.0 \\
Factor $f=1.0$ & 21.4 & 26{,}870 & 79.3 \\
Matched factor $f=0.75$ & 24.0 & 29{,}305 & 73.9 \\
Fr\'{e}chet $\delta=0.25$ & 25.3 & 26{,}169 & 82.8 \\
\quad + robust $\eta=0.04$ & 25.5 & 30{,}820 & 72.7 \\
\quad + robust $\eta=0.06$ & 25.7 & 34{,}694 & 63.7 \\
\quad + robust $\eta=0.08$ & 25.7 & 38{,}709 & 56.1 \\
\bottomrule
\end{tabular}
\end{table}

Two things are directly visible. First, even against the aggressive factor
$f=1.0$, Fr\'{e}chet dominates on both axes (25.3 versus 21.4 pass@1, 26{,}169
versus 26{,}870 NFE), and also edges matched factor $f=0.75$ (25.3 versus 24.0)
at 11\% fewer NFE. Second, robust confidence lifts pass@1 monotonically to 25.7
at $\eta=.06$ ($+4.3$ over factor $f=1.0$) by committing more conservatively, at
essentially threshold-level compute (34{,}694 versus 34{,}975 NFE). To separate
the failure modes we measured two quantities. Commit error counts commits that
disagree with the one-at-a-time reference. The argmax-flip probe is a direct
coupling signal: reveal a position's selected neighbors and recompute its
argmax; a flip means its top token changed, i.e.\ dependence the marginal
certificate cannot see. Commit error falls from 4.17\% (factor $f=1.0$) to
2.30\% (Fr\'{e}chet) to 0.46\% (robust $\eta=.06$); MBPP's flip rate under
Fr\'{e}chet is 0.4\%, no higher than GSM8K's 0.6\% and below factor $f=1.0$'s
1.0\%. So the drop is dominated by recoverable overconfidence, which the robust
variant corrects exactly as Corollary~\ref{cor:calibration-penalty} predicts;
residual coupling is small and not MBPP-specific, and this overconfidence hits
threshold and factor at least as hard.

\section{Generality: Backbones, Open-Ended Generation, and Length}\label{app:generality}

\paragraph{LLaDA-1.5.}
Table~\ref{tab:app-llada15} reports GSM8K (block 32) on
LLaDA-1.5~\citep{zhu2025llada15} across generation lengths and cache modes
with the five selector settings of Table~\ref{tab:main-results} and the matched factor. On the stronger LLaDA-1.5 backbone, the behavior at corresponding operating
points is unchanged: Fr\'{e}chet $\delta=0.25$ matches factor $f=0.75$ within
run-to-run noise (accuracy deltas between $-3.7$ and $+1.3$ points across the
six settings, and it wins outright in two) while using 10 to 18\% fewer NFE and
delivering 6 to 21\% higher throughput; the most aggressive pair ($\delta=0$
versus $f=1.0$) cuts NFE by 12 to 21\% at a 1 to 4 point cost. Threshold
$\tau=0.9$ remains the accuracy ceiling but is the slowest throughout (about
$1.8\times$ slower than Fr\'{e}chet $\delta=0.25$ at 256/PrefixCache), and the
NFE saving shrinks with generation length at fixed block size, the same trend
as base LLaDA.

\begin{table}[h]
\centering\small
\setlength{\tabcolsep}{4pt}
\caption{LLaDA-1.5, GSM8K, block 32. Each cell: accuracy / tok/s / NFE per example.}
\label{tab:app-llada15}
\begin{adjustbox}{width=\linewidth,center}
\begin{tabular}{llccccc}
\toprule
Gen & Cache & Threshold $\tau$=.9 & Factor $f$=1.0 & Fr\'{e}chet $\delta$=0 & Factor $f$=.75 & Fr\'{e}chet $\delta$=.25 \\
\midrule
256 & PrefixCache & 82.7 / 53.9 / 77.2 & 80.7 / 91.9 / 57.6 & 77.0 / 112.8 / 45.8 & 82.0 / 82.4 / 64.7 & 79.7 / 97.7 / 53.2 \\
256 & DualCache & 80.0 / 40.7 / 83.9 & 79.0 / 73.8 / 61.7 & 77.3 / 90.7 / 48.8 & 79.0 / 66.2 / 69.4 & 80.3 / 80.0 / 57.2 \\
512 & PrefixCache & 84.7 / 39.7 / 76.2 & 82.3 / 67.7 / 59.8 & 80.0 / 79.9 / 48.8 & 84.0 / 61.6 / 65.6 & 80.3 / 72.7 / 55.5 \\
512 & DualCache & 81.3 / 35.5 / 83.9 & 80.7 / 56.8 / 64.8 & 82.0 / 67.9 / 52.6 & 82.0 / 52.9 / 71.8 & 81.3 / 61.8 / 61.7 \\
1024 & PrefixCache & 82.3 / 19.5 / 81.8 & 81.7 / 28.3 / 68.3 & 79.0 / 30.7 / 60.4 & 82.3 / 26.7 / 72.7 & 81.0 / 28.9 / 65.5 \\
1024 & DualCache & 79.7 / 19.8 / 85.6 & 80.3 / 27.3 / 70.3 & 79.0 / 29.4 / 62.1 & 81.3 / 26.6 / 75.8 & 80.7 / 28.2 / 68.0 \\
\bottomrule
\end{tabular}
\end{adjustbox}
\end{table}

\paragraph{Open-ended generation: IFEval.}
IFEval is an open-ended generation benchmark (verifiable instruction following,
zero-shot, LLaDA-8B, PrefixCache); Table~\ref{tab:app-ifeval} reports the
results. Threshold is slightly more accurate but the slowest. At the
corresponding operating point ($\delta=0.25$ versus $f=0.75$),
instruction-following accuracy is identical on both metrics while Fr\'{e}chet
uses 4.0\% fewer NFE and is 6.5\% faster; against threshold, Fr\'{e}chet
$\delta=0.25$ gives up 0.7/0.5 points for 15.9\% fewer NFE and 26.5\% higher
throughput. The efficiency gain transfers to open-ended generation with
accuracy preserved.

\begin{table}[h]
\centering\small
\caption{IFEval, zero-shot, LLaDA-8B, PrefixCache.}
\label{tab:app-ifeval}
\begin{tabular}{lcccc}
\toprule
Method & Prompt-level strict & Instruction-level strict & NFE/example & tok/s \\
\midrule
Threshold $\tau=0.9$ & 58.0 & 66.5 & 216.7 & 33.9 \\
Factor $f=1.0$ & 56.3 & 65.8 & 179.8 & 42.3 \\
Fr\'{e}chet $\delta=0$ & 55.7 & 64.5 & 165.8 & 46.1 \\
Factor $f=0.75$ & 57.3 & 66.0 & 189.8 & 40.3 \\
Fr\'{e}chet $\delta=0.25$ & 57.3 & 66.0 & 182.2 & 42.9 \\
\bottomrule
\end{tabular}
\end{table}

\paragraph{Long generation.}
The gain also persists at long generation: GSM8K at lengths 1024 and 2048
shows 14.0\% and 10.0\% NFE reduction at $+0.7$ and $-0.3$ accuracy points
(Fr\'{e}chet $\delta=0.25$ versus factor $f=0.75$). The heterogeneity bonus
depends on the block size rather than the generation length
(Figure~\ref{fig:block-ablation}, Appendix~\ref{app:v2}).

\section{Fast-dLLM-v2 Sub-Block Transfer}\label{app:v2}

Fast-dLLM-v2~\citep{wu2025fastdllmv2} adapts a pretrained autoregressive model
into block diffusion and reuses confidence-aware parallel decoding for
intra-sub-block refinement, so Fr\'{e}chet profile decoding applies directly at
the sub-block selection step. We inserted Fr\'{e}chet only at Fast-dLLM-v2's
intra-sub-block commit-selection step on the public 7B checkpoint, leaving the
block-diffusion process unchanged, and ran the grid of sub-blocks
$\{4,8,16,32\}$ and $\delta\in\{0.10,0.25,0.30\}$. Table~\ref{tab:app-v2-native}
gives the unmodified native v2 decoder (its default threshold-0.9 selector);
Table~\ref{tab:app-v2} reports accuracy change in absolute points, NFE cut, and
throughput gain relative to it.

\begin{table}[h]
\centering\small
\caption{Fast-dLLM-v2 7B, native decoder (default threshold-0.9 selector): accuracy / NFE by sub-block size.}
\label{tab:app-v2-native}
\begin{tabular}{lcccc}
\toprule
Task & Sub-block 4 & Sub-block 8 & Sub-block 16 & Sub-block 32 \\
\midrule
GSM8K & 39.0 / 13{,}090 & 38.8 / 11{,}079 & 39.2 / 9{,}675 & 37.8 / 8{,}479 \\
MATH & 22.6 / 109{,}473 & 20.8 / 97{,}443 & 22.0 / 86{,}960 & 22.8 / 78{,}604 \\
\bottomrule
\end{tabular}
\end{table}

\begin{table}[h]
\centering\small
\caption{Fast-dLLM-v2 7B: Fr\'{e}chet inserted at the intra-sub-block selection step, relative to the native decoder of Table~\ref{tab:app-v2-native}.}
\label{tab:app-v2}
\begin{tabular}{llccc}
\toprule
Task & Sub-block & Acc change (abs pts) & NFE cut & tok/s gain \\
\midrule
\multicolumn{5}{l}{\emph{Fr\'{e}chet $\delta=0.25$ (default)}} \\
GSM8K & 4 & $-1.0$ & 19.7\% & $+24.5\%$ \\
GSM8K & 8 & $-3.0$ & 23.5\% & $+30.7\%$ \\
GSM8K & 16 & $-1.6$ & 25.4\% & $+34.1\%$ \\
GSM8K & 32 & $-1.2$ & 26.0\% & $+35.2\%$ \\
MATH & 4 & $-0.2$ & 17.0\% & $+20.5\%$ \\
MATH & 8 & $+0.0$ & 20.9\% & $+26.4\%$ \\
MATH & 16 & $-0.6$ & 22.7\% & $+29.3\%$ \\
MATH & 32 & $-3.8$ & 23.7\% & $+31.1\%$ \\
\midrule
\multicolumn{5}{l}{\emph{Fr\'{e}chet $\delta=0.10$ (aggressive)}} \\
GSM8K & 4 & $-0.8$ & 24.8\% & $+32.9\%$ \\
GSM8K & 8 & $-2.8$ & 28.0\% & $+38.8\%$ \\
GSM8K & 16 & $-3.0$ & 31.2\% & $+45.2\%$ \\
GSM8K & 32 & $-2.2$ & 31.9\% & $+46.8\%$ \\
MATH & 4 & $-0.8$ & 22.0\% & $+28.2\%$ \\
MATH & 8 & $-0.8$ & 25.9\% & $+35.0\%$ \\
MATH & 16 & $-1.8$ & 28.0\% & $+38.9\%$ \\
MATH & 32 & $-3.2$ & 29.5\% & $+41.8\%$ \\
\midrule
\multicolumn{5}{l}{\emph{Fr\'{e}chet $\delta=0.30$ (conservative)}} \\
GSM8K & 4 & $-0.4$ & 17.9\% & $+21.9\%$ \\
GSM8K & 8 & $-2.6$ & 22.0\% & $+28.2\%$ \\
GSM8K & 16 & $-1.0$ & 24.3\% & $+32.1\%$ \\
GSM8K & 32 & $-0.8$ & 25.9\% & $+35.0\%$ \\
MATH & 4 & $-0.6$ & 15.6\% & $+18.5\%$ \\
MATH & 8 & $-1.0$ & 19.0\% & $+23.5\%$ \\
MATH & 16 & $-0.8$ & 21.1\% & $+26.7\%$ \\
MATH & 32 & $-4.2$ & 22.0\% & $+28.2\%$ \\
\bottomrule
\end{tabular}
\end{table}

Even at the smallest sub-block (4) the bonus is meaningful: at the paper's
default $\delta=0.25$, Fr\'{e}chet reduces NFE by 17 to 20\% and improves
throughput by 20 to 25\% on GSM8K and MATH, with accuracy changes between
$-0.2$ and $-1.0$ points. The NFE gain grows monotonically with sub-block size
on every task (larger candidate sets expose more profile heterogeneity), and so
does the quality risk under more aggressive settings. These results show that
profile heterogeneity remains useful even in small block-diffusion candidate
sets. Table~\ref{tab:app-matched-sweeps} gives the matched comparison against
factor $f=1-\delta$ at sub-block size 4. These runs change only the selector
inside the block-diffusion decoder and keep the v2 cache configuration fixed.

\section{Compatibility with Concurrent Acceleration Methods}\label{app:composition}

\begin{table}[h]
\centering
\small
\setlength{\tabcolsep}{4pt}
\caption{\textbf{Fr\'{e}chet applied on top of recent complementary accelerators} (LLaDA-8B-Instruct, GSM8K, one H100, unless noted). Each method's own pipeline, checkpoint, cache, and harness are unchanged; only its commit rule is replaced by Fr\'{e}chet at the stated $\delta$, with efficiency in the pipeline's own counter. Appendix~\ref{app:composition} reports all margins and methods.}
\label{tab:composition}
\begin{adjustbox}{width=\linewidth,center}
\begin{tabular}{llccc}
\toprule
Method & Venue & Acc. (own rule) & Acc. (Fr\'{e}chet, $\delta$) & Efficiency change \\
\midrule
KLASS~\citep{kim2025klass} (same-setting selector) & NeurIPS 25 & 77.3 & 77.7 (.25) & NFE $-52\%$, tok/s $2.0\times$ \\
dLLM-Cache~\citep{liu2025dllmcache} & ICML 26 & 80.0 & 79.7 (.40) & NFE $-65\%$, tok/s $2.7\times$ \\
Streaming-dLLM~\citep{streaming2026} & arXiv 26 & 79.0 & 80.3 (.25) & NFE $-25\%$, tok/s $+25\%$ \\
Dynamic-dLLM~\citep{dynamicdllm2026} & ICLR 26 & 75.3 & 76.0 (.25) & NFE $-56\%$, tok/s $2.5\times$ \\
DSB~\citep{dsb2026} & ICML 26 & 81.3 & 81.0 (.30) & NFE $-32\%$, tok/s $1.4\times$ \\
DC-Leap~\citep{dcleap2026} & ICML 26 & 77.0 & 78.7 (.10) & tok/s 95.9 $\to$ 97.6 \\
Swordsman~\citep{swordsman2026} & ICML 26 & 79.3 & 80.7 (.10) & wall time $-16\%$ \\
CCD~\citep{ccd2026} (Dream-7B) & ICML 26 & 49.0 & 49.0 (.25) & steps $-54\%$, tok/s $2.1\times$ \\
dParallel~\citep{chen2025dparallel} (distilled) & ICLR 26 & 79.0 & 78.7 (.25) & NFE $-15\%$ vs matched factor \\
\bottomrule
\end{tabular}
\end{adjustbox}
\end{table}

\paragraph{Protocol.}
Each method's own pipeline, checkpoint, cache setting, and evaluation harness
are left unchanged; only its commit or acceptance rule is replaced, in place,
by the Fr\'{e}chet certificate. Unless noted: LLaDA-8B-Instruct, GSM8K 5-shot
flexible-extract, one dedicated H100. Efficiency counters use each pipeline's
own definition, identical in both arms of a block and never compared across
blocks. Each block gives the baseline's own numbers, then Fr\'{e}chet at several
margins; NFE cut is relative to the baseline row of each block unless stated
otherwise.

\paragraph{KLASS~\citep{kim2025klass}: adaptive KL-guided selector, same backbone and cache.}
KLASS can be evaluated under the same setting (LLaDA-8B, PrefixCache, GSM8K).
We ported KLASS's selector (confidence 0.6, KL 0.015, history 2) into the
identical harness and validated the port against the official implementation.
Every Fr\'{e}chet margin dominates KLASS on the accuracy-throughput frontier; at
$\delta=0.25$ Fr\'{e}chet obtains 77.7 accuracy and 102.4 tok/s, compared with
77.3 and 51.7 for KLASS, approximately twice the throughput and a 52\% NFE
reduction at comparable accuracy.
\begin{center}\small
\begin{tabular}{lcccc}
\toprule
 & acc & tok/s & NFE & tok/NFE \\
\midrule
KLASS (0.6/0.015/2) & 77.3 & 51.7 & 33{,}456 & 2.06 \\
\midrule
Fr\'{e}chet $\delta=0.0$ & 75.0 & 118.6 & $-59\%$ & 5.02 \\
Fr\'{e}chet $\delta=0.1$ & 75.0 & 113.8 & $-56\%$ & 4.73 \\
Fr\'{e}chet $\delta=0.2$ & 78.7 & 105.5 & $-53\%$ & 4.43 \\
Fr\'{e}chet $\delta=0.25$ & 77.7 & 102.4 & $-52\%$ & 4.28 \\
Fr\'{e}chet $\delta=0.3$ & 78.7 & 98.7 & $-50\%$ & 4.14 \\
\bottomrule
\end{tabular}
\end{center}

\paragraph{dLLM-Cache~\citep{liu2025dllmcache}: adaptive caching.}
The base row is the unmodified pipeline with its shipped decoding. Adding
Fr\'{e}chet multiplies dLLM-Cache's throughput by up to $3.2\times$; the
conservative $\delta=0.4$ setting retains nearly the baseline accuracy (79.7
versus 80.0) while reducing NFE by 65\%.
\begin{center}\small
\begin{tabular}{lcccc}
\toprule
 & acc & tok/s & NFE & tok/NFE \\
\midrule
dLLM-Cache, unmodified & 80.0 & 27.5 & 38{,}400 & 1.82 \\
\midrule
Fr\'{e}chet $\delta=0.1$ & 73.0 & 88.3 & $-71\%$ & 6.36 \\
Fr\'{e}chet $\delta=0.2$ & 75.0 & 82.9 & $-69\%$ & 5.99 \\
Fr\'{e}chet $\delta=0.25$ & 76.0 & 79.9 & $-68\%$ & 5.75 \\
Fr\'{e}chet $\delta=0.3$ & 78.3 & 78.8 & $-67\%$ & 5.59 \\
Fr\'{e}chet $\delta=0.4$ & 79.7 & 73.5 & $-65\%$ & 5.26 \\
\bottomrule
\end{tabular}
\end{center}

\paragraph{Streaming-dLLM~\citep{streaming2026}: suffix pruning.}
The base is threshold 0.9 plus its dynamic pruning. At $\delta=0.25$ Fr\'{e}chet
improves the unmodified pipeline on both axes at once, $+1.3$ accuracy points
and $+25\%$ throughput; suffix pruning does not subsume the selector gain.
\begin{center}\small
\begin{tabular}{lcccc}
\toprule
 & acc & tok/s & NFE & tok/NFE \\
\midrule
Streaming-dLLM, unmodified & 79.0 & 76.2 & 16{,}190 & 2.66 \\
\midrule
Fr\'{e}chet $\delta=0.1$ & 77.3 & 103.6 & $-33\%$ & 4.00 \\
Fr\'{e}chet $\delta=0.25$ & 80.3 & 95.0 & $-25\%$ & 3.59 \\
Fr\'{e}chet $\delta=0.3$ & 79.7 & 92.9 & $-23\%$ & 3.47 \\
Fr\'{e}chet $\delta=0.4$ & 77.7 & 88.4 & $-18\%$ & 3.23 \\
\bottomrule
\end{tabular}
\end{center}

\paragraph{Dynamic-dLLM~\citep{dynamicdllm2026}: dynamic cache budget and adaptive parallel decoding.}
Full native pipeline with the dynamic feature cache on; we replace only its
prediction-dynamics commit rule. At $\delta=0.25$ Fr\'{e}chet improves the
unmodified pipeline on both axes at once, $+0.7$ accuracy points at $2.5\times$
the throughput and $2.3\times$ fewer NFE; the dynamic feature cache does not
subsume the selector gain.
\begin{center}\small
\begin{tabular}{lcccc}
\toprule
 & acc & tok/s & NFE & tok/NFE \\
\midrule
Dynamic-dLLM, unmodified & 75.3 & 21.7 & 42{,}955 & 1.79 \\
\midrule
Fr\'{e}chet $\delta=0.1$ & 72.0 & 58.4 & $-60\%$ & 4.48 \\
Fr\'{e}chet $\delta=0.2$ & 74.0 & 54.5 & $-57\%$ & 4.19 \\
Fr\'{e}chet $\delta=0.25$ & 76.0 & 53.5 & $-56\%$ & 4.07 \\
Fr\'{e}chet $\delta=0.3$ & 73.3 & 50.4 & $-55\%$ & 3.94 \\
Fr\'{e}chet $\delta=0.4$ & 73.3 & 48.8 & $-51\%$ & 3.64 \\
\bottomrule
\end{tabular}
\end{center}

\paragraph{DSB~\citep{dsb2026}: dynamic sliding-block scheduling.}
Dynamic sliding-block scheduling with its cache, threshold accept rule replaced
in place (cache identical in both arms; NFE and tok/s from the repository's own
counters). At $\delta=0.3$ the certificate matches DSB's native accuracy within
noise (81.0 versus 81.3) at 32\% fewer NFE and $1.4\times$ tok/s; dynamic
blocking and the certificate compose.
\begin{center}\small
\begin{tabular}{lcccc}
\toprule
 & acc & tok/s & NFE & tok/NFE \\
\midrule
DSB, unmodified (thr .9) & 81.3 & 67.1 & 22{,}794 & 3.00 \\
\midrule
Fr\'{e}chet $\delta=0.1$ & 79.0 & 112.8 & $-42\%$ & 5.22 \\
Fr\'{e}chet $\delta=0.2$ & 79.7 & 105.3 & $-37\%$ & 4.80 \\
Fr\'{e}chet $\delta=0.25$ & 78.3 & 103.9 & $-35\%$ & 4.63 \\
Fr\'{e}chet $\delta=0.3$ & 81.0 & 95.1 & $-32\%$ & 4.43 \\
Fr\'{e}chet $\delta=0.4$ & 77.7 & 88.6 & $-26\%$ & 4.08 \\
\bottomrule
\end{tabular}
\end{center}

\paragraph{DC-Leap~\citep{dcleap2026}: draft-guided contiguous leaping.}
Draft-guided contiguous leaping unchanged; commit-threshold rule replaced. At
$\delta=0.1$ the certificate adds $+1.7$ accuracy points at slightly higher
throughput (97.6 versus 95.9 tok/s), confirming composition with draft-guided
leaping.
\begin{center}\small
\begin{tabular}{lccc}
\toprule
 & acc & tok/s & tok/NFE \\
\midrule
DC-Leap, unmodified (commit thr .70) & 77.0 & 95.9 & \\
\midrule
Fr\'{e}chet $\delta=0.1$ & 78.7 & 97.6 & 4.87 \\
Fr\'{e}chet $\delta=0.2$ & 77.7 & 90.1 & 4.54 \\
Fr\'{e}chet $\delta=0.25$ & 77.7 & 87.2 & 4.41 \\
Fr\'{e}chet $\delta=0.3$ & 77.3 & 84.7 & 4.27 \\
Fr\'{e}chet $\delta=0.4$ & 76.3 & 79.4 & 4.01 \\
\bottomrule
\end{tabular}
\end{center}

\paragraph{Swordsman~\citep{swordsman2026}: entropy-driven adaptive block partition.}
Entropy-driven adaptive block partition, native accept rule replaced (prefix
cache in both arms; wall time is the efficiency column, since its harness emits
no NFE counter). At $\delta=0.1$ the certificate adds $+1.3$ accuracy points at
16\% less wall time; adaptive block partition and the certificate are additive
on both axes.
\begin{center}\small
\begin{tabular}{lccc}
\toprule
 & acc & wall (s) & wall cut \\
\midrule
Swordsman, unmodified (thr .9) & 79.3 & 1{,}044 & \\
\midrule
Fr\'{e}chet $\delta=0.1$ & 80.7 & 881 & 16\% \\
Fr\'{e}chet $\delta=0.25$ & 79.0 & 905 & 13\% \\
Fr\'{e}chet $\delta=0.4$ & 79.0 & 989 & 5\% \\
\bottomrule
\end{tabular}
\end{center}

\paragraph{CCD~\citep{ccd2026}: coherent contextual decoding.}
Its released code is Dream-only, so this block is on Dream-7B-Instruct with
CCD's own counters (within-CCD comparison, not comparable to the LLaDA rows). On
the Dream backbone the certificate halves CCD's finish-steps and doubles its
answer throughput at equal or better accuracy at every margin tested, showing
the mechanism is not LLaDA-specific.
\begin{center}\small
\begin{tabular}{lccc}
\toprule
 & acc & tok/s & step cut \\
\midrule
CCD, unmodified (seek, expand 4.0) & 49.0 & 23.1 & (102.9 avg finish-steps) \\
\midrule
Fr\'{e}chet $\delta=0.1$ & 49.3 & 48.2 & 55\% \\
Fr\'{e}chet $\delta=0.25$ & 49.0 & 47.8 & 54\% \\
Fr\'{e}chet $\delta=0.4$ & 50.7 & 46.2 & 53\% \\
\bottomrule
\end{tabular}
\end{center}

\paragraph{WINO~\citep{wino2026}: revocable draft-and-verify.}
Revocable draft-and-verify; Fr\'{e}chet as the draft-acceptance rule,
verification loop untouched. Accuracy parity at every margin but steps rise:
WINO's revocable draft-verify loop is already more aggressive per step than a
certificate that never rolls back a commit. The transfer is safe but not
speed-additive; the two mechanisms address the same parallelism from opposite
directions (optimistic revoke versus certified commit).
\begin{center}\small
\begin{tabular}{lcc}
\toprule
 & acc & avg steps \\
\midrule
WINO, unmodified (thr .6) & 77.7 & 46.2 \\
\midrule
Fr\'{e}chet $\delta=0.1$ & 77.0 & 53.3 ($+15\%$) \\
Fr\'{e}chet $\delta=0.2$ & 76.0 & 56.2 ($+22\%$) \\
Fr\'{e}chet $\delta=0.25$ & 78.0 & 58.0 ($+26\%$) \\
Fr\'{e}chet $\delta=0.3$ & 77.3 & 60.4 ($+31\%$) \\
Fr\'{e}chet $\delta=0.4$ & 76.0 & 65.8 ($+42\%$) \\
\bottomrule
\end{tabular}
\end{center}

\paragraph{SOAR~\citep{soar2026}: confidence-switched position beam search.}
Confidence-switched position beam search (zero-shot chain-of-thought GSM8K per
its released scripts); its accept rule replaced. Forward work charges every
beam row, so this is a quality/compute frontier comparison, not a pure speedup
claim: the certificate cuts total forward work 55 to 65\% and roughly doubles
throughput at a 3 to 5 point accuracy cost. SOAR's beam search deliberately
spends extra compute for quality, so the two occupy different points of the
quality/compute frontier.
\begin{center}\small
\begin{tabular}{lcccc}
\toprule
 & acc & tok/s & forward work & tok/NFE \\
\midrule
SOAR, unmodified (conf .95, beam 2) & 72.7 & 70.2 & 39{,}995 & 1.82 \\
\midrule
Fr\'{e}chet $\delta=0.1$ & 68.7 & 157.1 & $-65\%$ & 5.17 \\
Fr\'{e}chet $\delta=0.2$ & 69.7 & 149.7 & $-62\%$ & 4.81 \\
Fr\'{e}chet $\delta=0.25$ & 68.0 & 151.0 & $-61\%$ & 4.61 \\
Fr\'{e}chet $\delta=0.3$ & 67.3 & 140.1 & $-59\%$ & 4.41 \\
Fr\'{e}chet $\delta=0.4$ & 69.0 & 139.5 & $-55\%$ & 4.05 \\
\bottomrule
\end{tabular}
\end{center}

\paragraph{Information-gain sampler~\citep{infogain2026}.}
Eight-candidate information-gain lookahead selection, replaced in place by the
Fr\'{e}chet certificate (cache-free in both arms, the only configuration in which
its released sampler is active; the efficiency axis is effective forward work,
since an eight-candidate lookahead batch charges eight rows). Replacing the
candidate-lookahead policy with the certificate cuts effective forward work by
94 to 95\% (10 to $13\times$ throughput) at a 1.7 to 3.7 point accuracy cost;
like SOAR, information gain deliberately spends lookahead compute for quality,
so the two occupy different points of the quality/compute frontier.
\begin{center}\small
\begin{tabular}{lccc}
\toprule
 & acc & tok/s & forward work \\
\midrule
Info-Gain, unmodified (8 candidates, thr .8) & 75.7 & 10.7 & 320{,}322 \\
\midrule
Fr\'{e}chet $\delta=0.1$ & 72.0 & 139.7 & $-95\%$ \\
Fr\'{e}chet $\delta=0.2$ & 74.0 & 138.0 & $-95\%$ \\
Fr\'{e}chet $\delta=0.25$ & 73.3 & 129.6 & $-95\%$ \\
Fr\'{e}chet $\delta=0.3$ & 73.3 & 120.8 & $-94\%$ \\
Fr\'{e}chet $\delta=0.4$ & 72.0 & 111.7 & $-94\%$ \\
\bottomrule
\end{tabular}
\end{center}

\paragraph{ME-DLM~\citep{medlm2026}: edit-based refinement.}
Edit-based refinement for parallel masked diffusion; its fixed per-step commit
quota in the mask-diffusion phase replaced by the Fr\'{e}chet certificate, the
edit-refinement phase untouched in both arms. This method requires its released
finetuned checkpoint and uses its own scorer, so the comparison is
within-method only. ME-DLM's default quota commits about 10.7 tokens per mask
step and pays for that aggressiveness in accuracy; the certificate commits
conservatively and adds 20.7 to 22.3 points of edit accuracy at every margin,
at 4.6 to $5.1\times$ the per-sample time. The same pattern holds on the full
1{,}319-example test set (63.0 versus 85.4 at $\delta=0.1$). This is a
quality-recovery trade within the ME-DLM harness, not a speedup claim.
\begin{center}\small
\begin{tabular}{lcc}
\toprule
 & edit acc & s/sample \\
\midrule
ME-DLM, unmodified (author quota) & 66.3 & 1.4 \\
\midrule
Fr\'{e}chet $\delta=0.1$ & 87.7 & 6.3 \\
Fr\'{e}chet $\delta=0.2$ & 87.0 & 6.5 \\
Fr\'{e}chet $\delta=0.25$ & 87.7 & 6.6 \\
Fr\'{e}chet $\delta=0.3$ & 88.3 & 6.7 \\
Fr\'{e}chet $\delta=0.4$ & 88.7 & 6.9 \\
\bottomrule
\end{tabular}
\end{center}

\paragraph{dParallel~\citep{chen2025dparallel}: learned parallel decoding (distilled checkpoint).}
Learned/distilled system (released checkpoint dParallel-LLaDA-8B-instruct, its
native certainty-forced threshold decoding). Fr\'{e}chet is swapped into the
distilled checkpoint; NFE cut is versus the matched factor selector
$f=1-\delta$ on the same checkpoint and the same decode loop (the like-for-like
selector baseline; the native decoder is a different generation method, so it is
not a comparable NFE baseline). dParallel is a complementary intervention (it
distills the model itself; accuracy shifts with distillation); the Fr\'{e}chet
efficiency mechanism still transfers onto its checkpoint, running 15 to 16\%
fewer NFE than matched factor.
\begin{center}\small
\begin{tabular}{lcccc}
\toprule
 & acc & tok/s & NFE & tok/NFE \\
\midrule
dParallel, native decoder & 79.0 & 50.6 & 9{,}842 & 6.11 \\
\midrule
 & acc & tok/s & NFE cut vs matched factor & tok/NFE \\
Fr\'{e}chet $\delta=0.1$ & 72.0 & 61.1 & 16\% & 7.38 \\
Fr\'{e}chet $\delta=0.2$ & 76.0 & 59.4 & 15\% & 7.17 \\
Fr\'{e}chet $\delta=0.25$ & 78.7 & 59.2 & 15\% & 7.15 \\
Fr\'{e}chet $\delta=0.3$ & 78.2 & 58.3 & 15\% & 7.04 \\
Fr\'{e}chet $\delta=0.4$ & 79.2 & 56.2 & 16\% & 6.78 \\
\bottomrule
\end{tabular}
\end{center}

\paragraph{d3LLM~\citep{d3llm2026}: pseudo-trajectory distillation (distilled checkpoint).}
Learned/distilled system (released checkpoint d3LLM\_LLaDA, its native
multi-block decoder; GSM8K under its own zero-shot chain-of-thought protocol).
Fr\'{e}chet is swapped into the distilled checkpoint. The native d3LLM decoder
wins here because d3LLM is distilled (pseudo-trajectory distillation with
early stopping) to parallel-decode along its own trained multi-block schedule,
so it already realizes most of the safe parallelism a training-free selector
targets, and that distillation sharpens per-token confidences, flattening the
heterogeneity Fr\'{e}chet's bonus uses. Fr\'{e}chet therefore has little room
for improvement over the native decoder.
\begin{center}\small
\begin{tabular}{lcccc}
\toprule
 & acc & tok/s & NFE & tok/NFE \\
\midrule
d3LLM, native decoder & 71.3 & 226.7 & 7{,}321 & 9.05 \\
\midrule
 & acc & tok/s & NFE cut vs matched factor & tok/NFE \\
Fr\'{e}chet $\delta=0.1$ & 60.0 & 224.5 & 21\% & 8.96 \\
Fr\'{e}chet $\delta=0.2$ & 66.0 & 221.8 & 22\% & 8.86 \\
Fr\'{e}chet $\delta=0.25$ & 64.0 & 212.1 & 21\% & 8.47 \\
Fr\'{e}chet $\delta=0.3$ & 69.0 & 207.8 & 20\% & 8.30 \\
Fr\'{e}chet $\delta=0.4$ & 69.0 & 195.2 & 20\% & 7.79 \\
\bottomrule
\end{tabular}
\end{center}

\paragraph{DAPD~\citep{dapd2026}: decoding without blocks.}
DAPD decodes on a blockless global 256-token canvas. Inserting the certificate
there is unsafe: it mass-commits the high-confidence end-of-sequence tail
before the reasoning completes, and accuracy collapses (73.0 to 24.3 at
$\delta=0.1$) while steps rise. We verified per sample that this is a
block-structure interaction, not a certificate defect: all positive results
above use block-structured decoding, and DAPD's own rule avoids the same
failure only by committing exclusively near-saturated confidences. The
Fr\'{e}chet certificate assumes block-structured or windowed candidate sets;
blockless global canvases require an end-of-sequence or window guard before it
applies. In the head-on reading, DAPD-Direct reaches 98.1 tok/s at 73.0
accuracy under its native rule, while Fr\'{e}chet $\delta=0$ on the same
checkpoint, examples, and hardware reaches 97.1 tok/s at 78.0 accuracy.

\paragraph{Methods discussed but not run.}
As of August 2026 the following had no usable public implementation, or could
not be run under an oracle-free same-backbone protocol, and are discussed in
Appendix~\ref{app:related-extended} instead: PUNT~\citep{azangulov2025parallel},
ETE~\citep{ete2026}, LSP~\citep{lsp2026}, locally coherent parallel
decoding~\citep{locallycoherent2026}, temporal-spatial confidence
extrapolation~\citep{temporalspatial2026}, and lookahead
unmasking~\citep{lookahead2026} (repositories absent or placeholders);
self-speculative masked diffusions~\citep{selfspec2026} additionally requires
retraining; the product-of-experts bridge~\citep{poebridge2026} releases code
but its verifier mixture requires a Qwen-vocabulary backbone (Dream only, no
LLaDA path); and Prophet~\citep{prophet2026} constitutively forces an answer
template and monitors the answer region, so an oracle-free scored comparison
is not defined.

\paragraph{Summary.}
Inserted into the concurrent training-free pipelines with usable code, the
certificate is additive on both axes for Dynamic-dLLM, DSB, DC-Leap, and
Swordsman (and, on Dream, CCD), as for KLASS, dLLM-Cache, and Streaming-dLLM;
it recovers over 20 accuracy points inside ME-DLM's over-aggressive quota
decoder; it is a quality/compute frontier trade against the deliberately
compute-heavy SOAR and information-gain samplers; it is safe but not
speed-additive inside WINO's revocable verifier; it transfers to the dParallel
checkpoint; the native d3LLM decoder wins on its distilled checkpoint; and it
requires block structure (DAPD). To our knowledge this covers every concurrent
ICLR/ICML 2026 training-free dLLM acceleration method with a usable public
implementation.

\section{Extended Related Work}\label{app:related-extended}

Concurrent training-free acceleration of dLLM inference spans several
mechanisms, grouped below by what each one changes.

\paragraph{Confidence-based selectors.}
KLASS~\citep{kim2025klass}, CCD~\citep{ccd2026}, temporal-spatial confidence
extrapolation~\citep{temporalspatial2026}, and lookahead
unmasking~\citep{lookahead2026} decide per-step commit sets from marginal
confidence signals. Fast-dLLM's threshold and factor rules~\citep{wu2025fastdllm}
belong to this family, and Fr\'{e}chet profile decoding strictly generalizes
the factor rule within it.

\paragraph{Dependence- and coupling-aware methods.}
DAPD~\citep{dapd2026}, PUNT~\citep{azangulov2025parallel},
CoDD~\citep{codd2026}, and locally coherent parallel
decoding~\citep{locallycoherent2026} use attention structure, independence
tests, or learned joint layers to model what a marginal factorization ignores.
Fr\'{e}chet decoding uses only the marginal confidences; Lemma~\ref{thm:tv-stability}
describes how such dependence information can certify additional commits.

\paragraph{Adaptive schedulers.}
SlowFast sampling~\citep{slowfast2026}, DUS~\citep{dus2026},
Dynamic-dLLM~\citep{dynamicdllm2026}, ETE~\citep{ete2026},
APD~\citep{israel2025apd}, and Prophet~\citep{prophet2026} reallocate where
and how many tokens are decoded over the trajectory.

\paragraph{Draft-and-verify and revocable decoders.}
WINO~\citep{wino2026}, DC-Leap~\citep{dcleap2026}, self-speculative masked
diffusions~\citep{selfspec2026}, and the product-of-experts
bridge~\citep{poebridge2026} trade extra verification compute for aggressive
drafting.

\paragraph{Cache and system methods.}
dLLM-Cache~\citep{liu2025dllmcache}, Streaming-dLLM~\citep{streaming2026},
DPad~\citep{chen2026dpad}, and stopping computation for converged
tokens~\citep{stopconverged2026} cut per-step cost without changing the
commit rule.

\paragraph{Adaptive block partition.}
Fast-dLLM-v2~\citep{wu2025fastdllmv2}, Swordsman~\citep{swordsman2026}, and
DSB~\citep{dsb2026} restructure the canvas into blocks or sub-blocks.

\paragraph{Learned or distilled decoders.}
dParallel~\citep{chen2025dparallel}, d3LLM~\citep{d3llm2026}, and
LSP~\citep{lsp2026} internalize parallelism at training time.

\paragraph{Theory and benchmarks.}
Recent work characterizes parallel-sampling limits, including provably optimal
parallel samplers~\citep{optimalsamplers2026} and
ParallelBench~\citep{kang2025parallelbench}.

Fr\'{e}chet profile decoding strictly generalizes the confidence-based factor
rule and is complementary to the caching, scheduling, blocking, and
draft-verify mechanisms above: applied on top of them it improves both quality
and efficiency (Appendix~\ref{app:composition}).

\clearpage
\section*{NeurIPS Paper Checklist}

\begin{enumerate}

\item {\bf Claims}
    \item[] Question: Do the main claims made in the abstract and introduction accurately reflect the paper's contributions and scope?
    \item[] Answer: \answerYes{} 
    \item[] Justification: As stated in abstract, our resulting rule is a heterogeneous-confidence generalization of Fast-dLLM’s factor selector. We draw the conclusion from comprehensive empirical experiments that this rule provides performance improvement compared to baselines.

\item {\bf Limitations}
    \item[] Question: Does the paper discuss the limitations of the work performed by the authors?
    \item[] Answer: \answerYes{} 
    \item[] Justification: We discuss limitations and future directions of our method in the end of main paper, including improved bounds given extra dependency information.

\item {\bf Theory assumptions and proofs}
    \item[] Question: For each theoretical result, does the paper provide the full set of assumptions and a complete (and correct) proof?
    \item[] Answer: \answerYes{} 
    \item[] Justification: Assumptions and proofs are given in main paper and appendix.

    \item {\bf Experimental result reproducibility}
    \item[] Question: Does the paper fully disclose all the information needed to reproduce the main experimental results of the paper to the extent that it affects the main claims and/or conclusions of the paper (regardless of whether the code and data are provided or not)?
    \item[] Answer: \answerYes{} 
    \item[] Justification: We detailed hyperparameters and provided the codebase for our implementation.

\item {\bf Open access to data and code}
    \item[] Question: Does the paper provide open access to the data and code, with sufficient instructions to faithfully reproduce the main experimental results, as described in supplemental material?
    \item[] Answer: \answerYes{} 
    \item[] Justification: We are providing the codebase of our implementation and experiments.

\item {\bf Experimental setting/details}
    \item[] Question: Does the paper specify all the training and test details (e.g., data splits, hyperparameters, how they were chosen, type of optimizer) necessary to understand the results?
    \item[] Answer: \answerYes{} 
    \item[] Justification: Although our method is training-free and plug-in, we detailed our choice of hyperparameters and models for replicating experiments in the paper. 

\item {\bf Experiment statistical significance}
    \item[] Question: Does the paper report error bars suitably and correctly defined or other appropriate information about the statistical significance of the experiments?
    \item[] Answer: \answerNo{} 
    \item[] Justification: Evaluation of our benchmarks involve heavy computational resource, therefore, producing statistical significance for our results goes beyond our capacity.

\item {\bf Experiments compute resources}
    \item[] Question: For each experiment, does the paper provide sufficient information on the computer resources (type of compute workers, memory, time of execution) needed to reproduce the experiments?
    \item[] Answer: \answerYes{} 
    \item[] Justification: We detailed experimental hardware requirements in experiment section.
    
\item {\bf Code of ethics}
    \item[] Question: Does the research conducted in the paper conform, in every respect, with the NeurIPS Code of Ethics \url{https://neurips.cc/public/EthicsGuidelines}?
    \item[] Answer: \answerYes{} 
    \item[] Justification: We have faithfully followed the guidelines provided by NeurIPS.

\item {\bf Broader impacts}
    \item[] Question: Does the paper discuss both potential positive societal impacts and negative societal impacts of the work performed?
    \item[] Answer: \answerNo{} 
    \item[] Justification: The work focuses on the run-time acceleration of diffusion language models and is not application-specific. Therefore, there is no specific societal impacts to be claimed.  
    
\item {\bf Safeguards}
    \item[] Question: Does the paper describe safeguards that have been put in place for responsible release of data or models that have a high risk for misuse (e.g., pre-trained language models, image generators, or scraped datasets)?
    \item[] Answer: \answerNA{} 
    \item[] Justification: Our project do not pose sch risks.

\item {\bf Licenses for existing assets}
    \item[] Question: Are the creators or original owners of assets (e.g., code, data, models), used in the paper, properly credited and are the license and terms of use explicitly mentioned and properly respected?
    \item[] Answer: \answerYes{} 
    \item[] Justification: Original assets are credited.

\item {\bf New assets}
    \item[] Question: Are new assets introduced in the paper well documented and is the documentation provided alongside the assets?
    \item[] Answer: \answerYes{} 
    \item[] Justification: We will provide a well-structured codebase as our original assets.

\item {\bf Crowdsourcing and research with human subjects}
    \item[] Question: For crowdsourcing experiments and research with human subjects, does the paper include the full text of instructions given to participants and screenshots, if applicable, as well as details about compensation (if any)? 
    \item[] Answer: \answerNA{} 
    \item[] Justification: Our experiment does not include engagement with human subjects.

\item {\bf Institutional review board (IRB) approvals or equivalent for research with human subjects}
    \item[] Question: Does the paper describe potential risks incurred by study participants, whether such risks were disclosed to the subjects, and whether Institutional Review Board (IRB) approvals (or an equivalent approval/review based on the requirements of your country or institution) were obtained?
    \item[] Answer: \answerNA{} 
    \item[] Justification: Our paper does not involve crowdsourcing nor research with human subjects.

\item {\bf Declaration of LLM usage}
    \item[] Question: Does the paper describe the usage of LLMs if it is an important, original, or non-standard component of the core methods in this research? Note that if the LLM is used only for writing, editing, or formatting purposes and does \emph{not} impact the core methodology, scientific rigor, or originality of the research, declaration is not required.
    \item[] Answer: \answerNA{} 
    \item[] Justification: Our core methodology - the theoretical bound of performance-competitive acceleration is developed without the assistance of LLM.

\end{enumerate}

\end{document}